\documentclass[]{databricksai}
\usepackage{bm}

\newcommand{\figref}[1]{Figure~\ref{#1}}

\usepackage{wrapfig}
\usepackage{makecell}
\usepackage{enumitem}
\usepackage{tikz}
\usepackage{listings}
\usepackage{algorithmicx}
\usepackage[noend]{algpseudocode}

\definecolor{accentblue}{HTML}{376DCC}
\definecolor{accentorange}{HTML}{EE5A24}

\newcommand{\secvsabove}{\vspace{-1.5mm}}

\lstdefinelanguage{json}{
  basicstyle=\ttfamily\footnotesize,
  numbers=left,
  stepnumber=1,
  numberstyle=\tiny,
  breaklines=true,
}

\lstdefinestyle{mdsrc}{
  basicstyle=\ttfamily\scriptsize,
  breaklines=true,
  numbers=left,
  numberstyle=\tiny
}

\tcbset{
  templatebox/.style={
    colback=dbxbg,
    colframe=black,
    enhanced,
    boxrule=0.6pt,
    arc=2pt,
    left=6pt,
    right=6pt,
    top=6pt,
    bottom=6pt
  }
}
\newtcolorbox{TemplateBox}[1][]{templatebox,#1}

\algrenewcommand\algorithmicrequire{\textbf{Input:}}
\algrenewcommand\algorithmicensure{\textbf{Output:}}

\usepackage[textsize=tiny]{todonotes}
\usepackage{draftfigure}
\usepackage{tabularx}
\usepackage{float}
\usepackage{amsthm}

\theoremstyle{plain}
\newtheorem{theorem}{Theorem}[section]
\newtheorem{proposition}[theorem]{Proposition}
\newtheorem{lemma}[theorem]{Lemma}

\theoremstyle{definition}

\theoremstyle{remark}

\title{Spend Your Rollouts Where It Counts: Rollout Allocation for Group-Based RL Post-Training}

\author[1,2]{Woojeong Kim}
\author[1]{Ziyi Yang}
\author[2]{Jing Nathan Yan}
\author[1]{Jialu Liu}

\affiliation[1]{Databricks}
\affiliation[2]{Cornell University}
\abstract{
Reinforcement learning (RL) is the dominant paradigm for post-training large language models. However, in the online, on-policy setting, rollout generation dominates the computational cost of training. Group-based policy optimization methods compute advantages from multiple rollouts per prompt, yet they indiscriminately allocate budget to prompts with collapsed reward distributions, wasting expensive rollouts on negligible learning signals. We demonstrate that group-based updates are most effective in regimes of high reward variance. Since the policy evolves throughout training, prompt informativeness must be estimated online rather than precomputed, but exhaustively evaluating every prompt is computationally prohibitive. We introduce \emph{Pilot-Commit}, a budget-aware rollout allocation framework for group-based RL post-training. Pilot-Commit decouples prompt evaluation from exploitation: a pilot stage estimates per-prompt informativeness using a fraction of the budget, and the remaining rollouts are allocated to high-leverage prompts while low-signal prompts are skipped. Across multiple math reasoning benchmarks and model scales from 1.5B to 14B parameters, Pilot-Commit matches baseline accuracy with significantly lower sampling costs, reaching target accuracy up to $1.9\times$ faster than GRPO and $4.0\times$ faster than DAPO in cumulative rollouts.
}

\date{\today}
\correspondence{Woojeong Kim \email{wk247@cornell.edu}, 
Jialu Liu \email{jialu.liu@databricks.com}}
\metadata[Code]{\url{https://github.com/databricks/pilot-commit}}

\begin{document}
\maketitle

\secvsabove

\begin{figure}[!ht]
    \centering
    \begin{subfigure}[t]{0.48\textwidth}
        \centering
        \includegraphics[width=\linewidth]{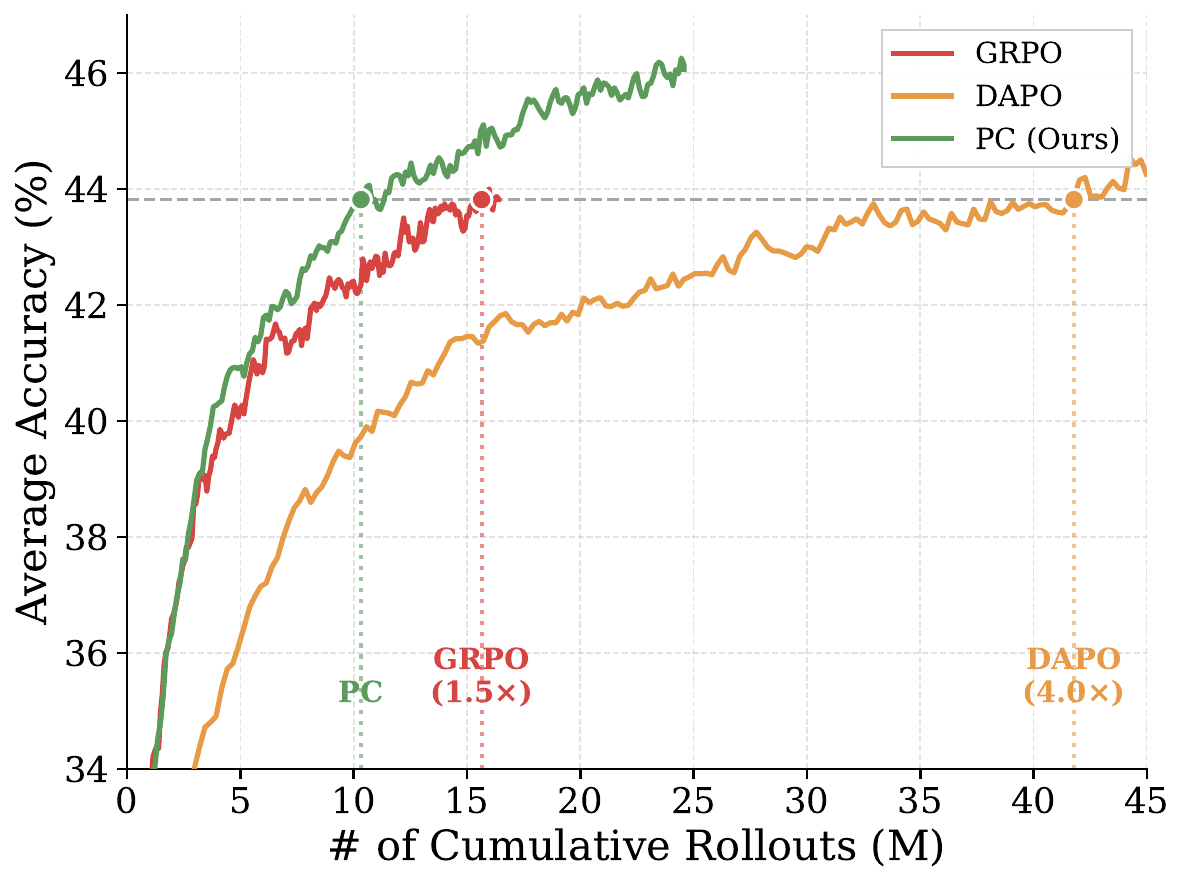}
        \caption{Qwen2.5-Math-1.5B, DeepMath-103K ($n{=}128$).}
        \label{fig:main-1.5b}
    \end{subfigure}
    \hfill
    \begin{subfigure}[t]{0.48\textwidth}
        \centering
        \includegraphics[width=\linewidth]{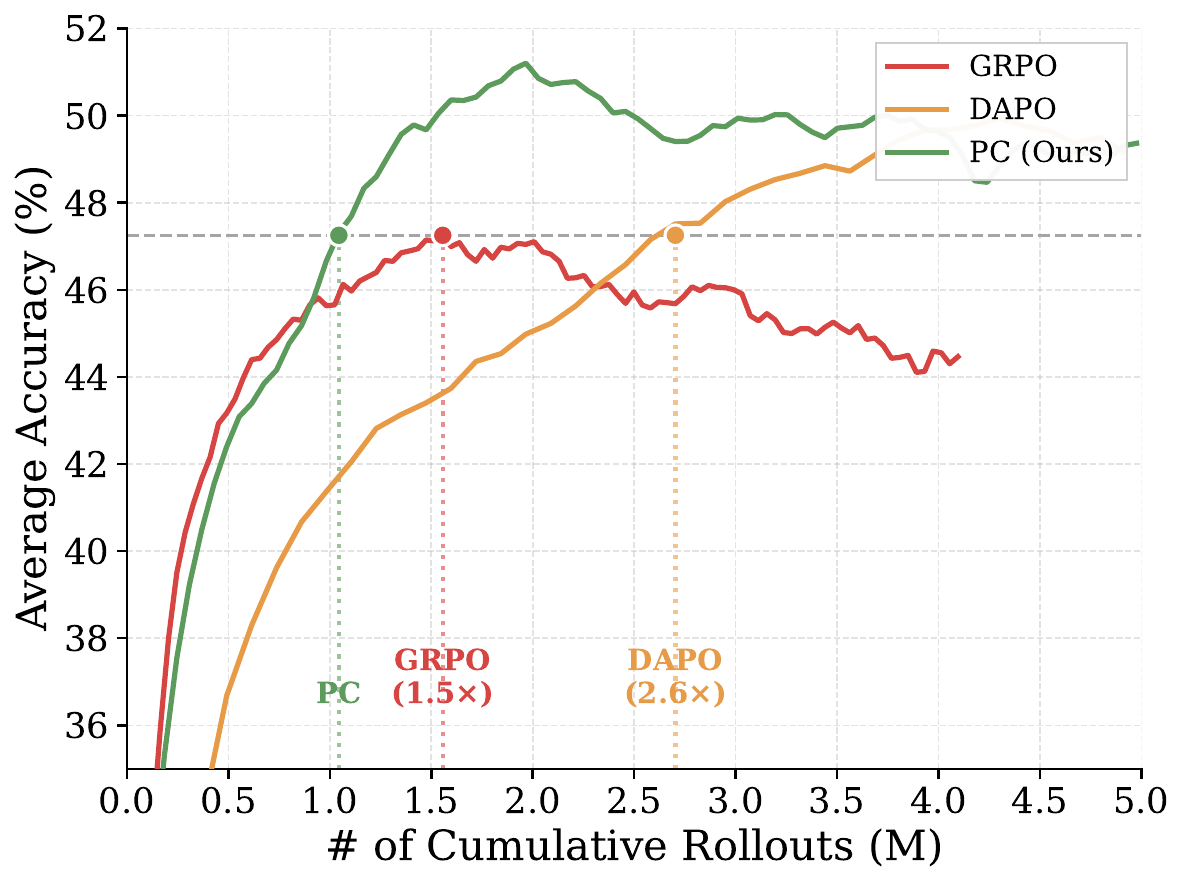}
        \caption{Qwen3-8B, Polaris-53K ($n{=}64$).}
        \label{fig:main-8b}
    \end{subfigure}
    \caption{
    Average accuracy across math benchmarks versus cumulative rollouts. Our proposed Pilot-Commit (PC) reaches baseline accuracy using $1.5\times$/$4.0\times$ fewer rollouts at 1.5B (a) and $1.5\times$/$2.6\times$ fewer at 8B (b) compared to GRPO/DAPO. Sampling cost per training step varies by method (DAPO $>$ PC $>$ GRPO). The two settings differ in total budget due to different group sizes and training durations.
    }
    \label{fig:main}
\end{figure}

\section{Introduction}
\label{sec:intro}

\begin{figure*}[!t]
    \centering
    \includegraphics[width=0.8\textwidth]{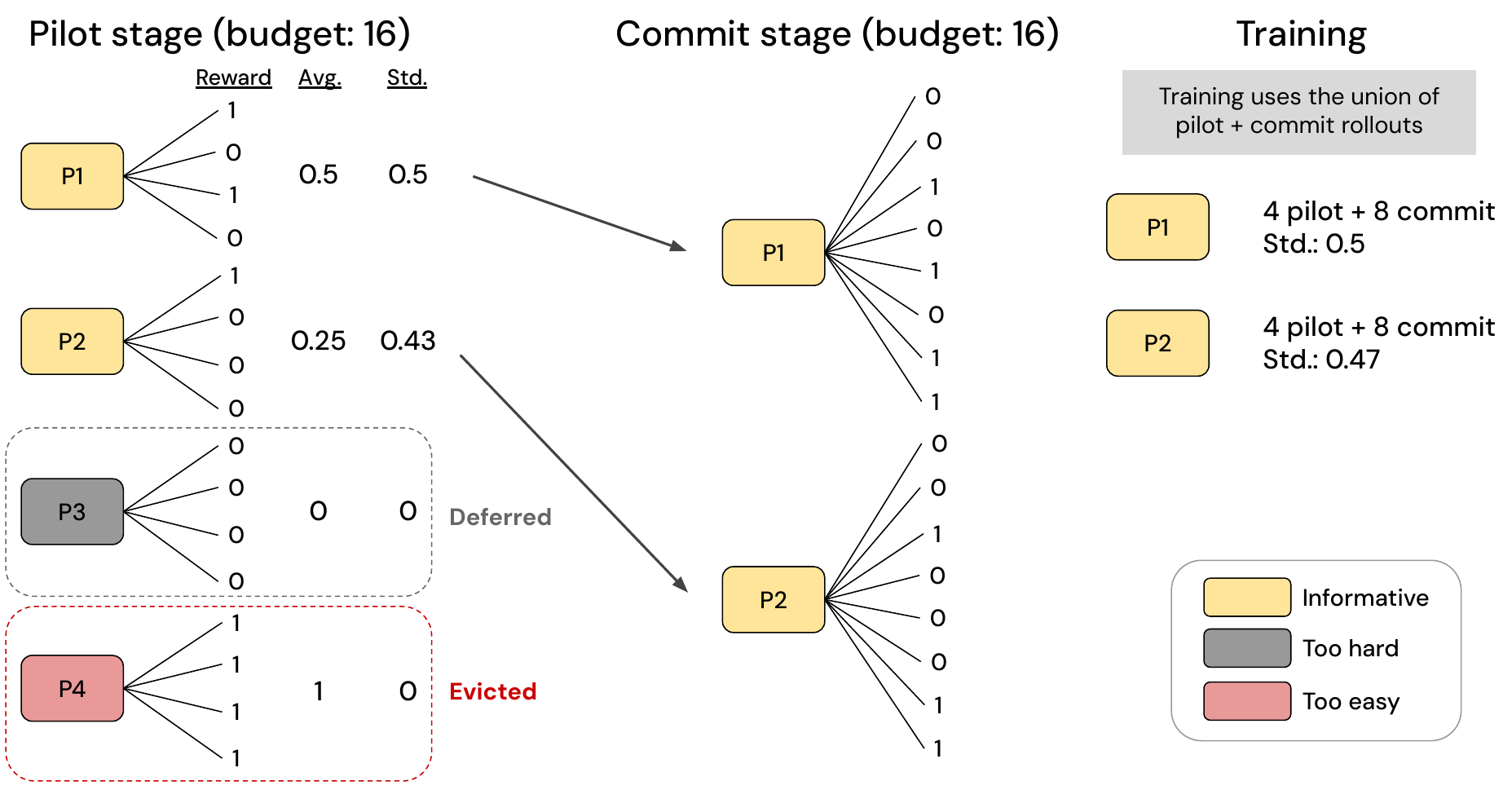}
    \caption{
     Overview of Pilot-Commit. Each box represents a prompt. In the pilot stage, a fraction of the rollout budget is used to estimate empirical reward variance per prompt. In the commit stage, the remaining budget is allocated to selected prompts ($P1$ and $P2$). The policy is then updated on the union of pilot and commit rollouts for the selected prompts. Prompt $P3$ is deemed too hard and deferred to the next epoch without proceeding to the commit stage. Prompt $P4$ is deemed too easy and evicted from future sampling.
    }
    \label{fig:pc-overview}
\end{figure*}

Large language models (LLMs) require post-training to adapt to specific objectives after pre-training~\citep{brown2020language, chowdhery2023palm}. Reinforcement learning (RL) has emerged as one of the most effective paradigms for this stage~\citep{schulman2017proximal, ouyang2022training, rafailov2023direct, shao2024deepseekmath}, enabling alignment with complex goals such as human preferences and mathematical reasoning.

However, online RL post-training introduces challenges that differ fundamentally from supervised fine-tuning (SFT)~\citep{yao2023deepspeed, mei2025real}. The training data distribution is not fixed but depends on the evolving policy, and inference is an integral part of the training loop, as rollouts must be generated online to compute rewards. As a result, the efficiency bottleneck in RL post-training is driven by sampling rather than optimization.

In group-based algorithms~\citep{shao2024deepseekmath, liu2025understanding}, advantages are computed relative to other rollouts from the same prompt, so the learning signal depends directly on reward diversity within each group. Prompts whose rollouts yield identical outcomes produce zero advantage and contribute nothing to the gradient, yet existing methods allocate the same budget to every prompt regardless. Oversampling strategies like DAPO~\citep{yu2025dapo} improve exploration but scale the sampling cost linearly, without distinguishing informative prompts from uninformative ones. We argue that rollout allocation is a first-class problem: given a fixed rollout budget per step, the question is not how many rollouts to generate, but which prompts to spend them on.

In this work, we propose \emph{Pilot-Commit} (PC), a rollout allocation framework for group-based RL post-training. In the \emph{pilot} stage, a fraction of the rollout budget is used to estimate reward diversity under the current policy. Prompts in the high reward-variance regime, where gradients are most informative, then receive the remaining \emph{commit} budget, while low-signal prompts are skipped. While the framework generalizes to continuous and multi-class rewards, we focus on verifiable binary rewards, where reward variance is a direct proxy for prompt informativeness.

A key distinction from curriculum learning is that Pilot-Commit does not impose a fixed ordering over prompt difficulty or modify the training objective. Instead, it adapts to the current policy online, reallocating rollouts toward prompts where the gradient signal is strongest under the evolving policy.

We show analytically that the GRPO gradient is maximized when reward outcomes are maximally uncertain. Since prompts become easier over the course of training, we apply asymmetric thresholds that retain hard prompts approaching this regime while skipping near-solved ones whose signal is diminishing. Skipped prompts are deferred to later epochs, not permanently removed. We additionally introduce an eviction strategy, a replay buffer, and a one-step pilot-commit binding to reduce overhead.

We evaluate Pilot-Commit on mathematical reasoning benchmarks across model scales from 1.5B to 14B parameters. As shown in \figref{fig:main}, PC consistently reaches baseline accuracy with fewer rollouts, using up to $1.9\times$ fewer than GRPO and $4.0\times$ fewer than DAPO across settings.
We implement Pilot-Commit as a new recipe in the verl~\citep{sheng2024hybridflow} RL framework and release the code alongside this paper.

\section{Related Work}
\label{sec:rw}

\paragraph{Policy Optimization for Language Models}
To reduce variance and engineering complexity associated with token-level value estimation in PPO~\citep{schulman2017proximal}, recent works replace explicit critic learning with group-based baselines.
GRPO~\citep{shao2024deepseekmath} computes advantages using multiple responses per prompt, eliminating the need for a learned critic.
Several extensions further explore group-based advantage estimation and sampling strategies~\citep{lin2025cppo, yu2025dapo, liu2025understanding}.

\paragraph{Inference Bottlenecks in RL Post-Training}
In online RL training, rollout generation often dominates the cost of gradient updates, making sampling the primary computational bottleneck.
Recent systems address this through high-throughput inference engines~\citep{kwon2023efficient, zheng2024sglang} and asynchronous training frameworks that overlap rollout generation with policy updates~\citep{mei2025real, wu2025llamarl}.
These system-level advances improve throughput but do not address the underlying inefficiency of allocating sampling budget to low-signal data. Our work is complementary, focusing on improving the effectiveness of each sampled rollout.

\paragraph{Prompt Selection for RL Post-Training}
Selectively focusing computation on informative examples is a well-studied idea in curriculum learning~\citep{bengio2009curriculum, kumar2010self}, active learning~\citep{lewis1994heterogeneous, settles2009active}, and hard example mining~\citep{shrivastava2016training, lin2017focal}.
In the context of LLM post-training, several recent methods adapt prompt selection to reduce rollout cost, differing primarily in how they estimate prompt informativeness.

One approach filters or reweights rollouts after generation~\citep{xu2504not, shrivastava2025sample}, which improves update quality but does not reduce the cost of rollout generation itself.
Another relies on statistics from previous epochs to reweight or skip prompts~\citep{zheng2025act, li2025knapsack, Polaris2025}, but these estimates can become stale as the policy evolves.
A third line trains auxiliary models to predict prompt difficulty: GPS~\citep{qu2025gps} uses a generative prompt predictor, MoPPS~\citep{qu2025mopps} maintains Beta posteriors over success rates, PDTS~\citep{qu2025pdts} applies posterior sampling with diversity regularization, and PCL~\citep{gao2025prompt} uses a learned value model updated concurrently with the policy.
Other works apply curriculum-style strategies based on reward variance~\citep{jiang2025vcrl, zhang2506speed}.

In contrast, Pilot-Commit estimates prompt informativeness online using pilot-stage feedback from the current policy, avoiding auxiliary models and stale historical estimates.
By reallocating a fixed rollout budget based on live reward variance, it reduces wasted sampling while preserving the training objective.

\section{Method}

\subsection{Preliminaries: GRPO}
\label{sec:grpo}

We consider a conditional language modeling setting in which a policy $\pi_\theta$, parameterized by $\theta$, generates an output sequence $o = (o_1, \dots, o_{|o|})$ conditioned on an input prompt $q$. At each time step $t$, the policy defines a conditional distribution $\pi_\theta(o_t \mid q, o_{<t})$ over the next token.

Policy optimization is formulated as a reinforcement learning problem. For each prompt $q \sim P(Q)$, we sample a group of $G$ output sequences $\{o_i\}_{i=1}^G$ independently from a fixed behavior policy $\pi_{\theta_{\mathrm{old}}}(\cdot \mid q)$. 
Updates are performed at the token level using importance sampling, leading to the Group Relative Policy Optimization (GRPO) objective with clipped policy gradients and KL regularization.

The GRPO objective is defined as
\begin{equation}
\begin{split}
\mathcal{J}(\theta) = &\mathbb{E}_{q, \{o_i\}_{i=1}^G} \Bigg[ 
\frac{1}{G}\sum_{i=1}^{G} \Bigg( \frac{1}{|o_i|} \sum_{t=1}^{|o_i|} \min \big( r_{i,t}(\theta) \hat{A}_{i,t}, \quad \text{clip}(r_{i,t}(\theta), 1 \pm \epsilon) \hat{A}_{i,t} \big) \Bigg)
- \beta \mathrm{D_{KL}}(\pi_\theta \| \pi_{\text{ref}}) \Bigg],
\end{split}
\end{equation}
where $\epsilon > 0$ is the clipping threshold, $\beta \ge 0$ is the KL regularization coefficient (set to 0 in our experiments), and $\pi_{\mathrm{ref}}$ is a fixed reference policy.

The token-level importance ratio is defined as
\begin{equation}
  r_{i,t}(\theta) =
\frac{\pi_\theta(o_{i,t} \mid q, o_{i,<t})}
{\pi_{\theta_{\mathrm{old}}}(o_{i,t} \mid q, o_{i,<t})}.  
\end{equation}

Each sampled sequence $o_i$ is assigned a scalar reward $r_i \in \mathbb{R}$. GRPO constructs advantages using group-relative normalization. For a fixed prompt $q$, let $\mu$ and $\sigma$ denote the within-group mean and standard deviation of the rewards:

\begin{equation}
\mu = \frac{1}{G} \sum_{i=1}^G r_i,
\quad
\sigma =
\sqrt{\frac{\sum_{i=1}^G (r_i - \mu)^2}{G}}.
\end{equation}

The group-relative advantage for sequence $o_i$ is $\hat{A}_i = (r_i - \mu)/\sigma$, broadcast to all tokens so that $\hat{A}_{i,t} = \hat{A}_i$ for all $t$.

\subsection{GRPO Learning Signal in the High Reward-Variance Regime}
\label{sec:grpo_high_variance}

We study when GRPO produces a strong group-wise policy-gradient signal.
To isolate the effect of group-relative advantages, we analyze a per-prompt surrogate that drops PPO clipping. Since $\beta=0$ in our experiments, this surrogate differs from the full objective only in the removal of clipping.

\paragraph{Per-prompt surrogate.}
For a fixed prompt $q$ with a sampled group $\{o_i\}_{i=1}^G \sim \pi_{\theta_{\mathrm{old}}}(\cdot\mid q)$, we consider
\begin{equation}
\mathcal{J}_{\mathrm{sur}}(\theta)
=
\frac{1}{G}\sum_{i=1}^G \frac{1}{|o_i|}\sum_{t=1}^{|o_i|}
r_{i,t}(\theta)\,\hat{A}_{i}, \quad \quad
\nabla_\theta \mathcal{J}_{\mathrm{sur}}(\theta)
=
\frac{1}{G}\sum_{i=1}^G \hat{A}_i\, g_i(\theta),
\end{equation}
where
\begin{equation}
    g_i(\theta) \;=\; \frac{1}{|o_i|}\sum_{t=1}^{|o_i|}
r_{i,t}(\theta)\nabla_\theta\log \pi_\theta(o_{i,t}\mid q,o_{i,<t}).
\end{equation}

We ask how the distribution of $\{\hat A_i\}_{i=1}^G$ affects the magnitude of the group gradient
$\lVert \nabla_\theta \mathcal{J}_{\mathrm{sur}}(\theta)\rVert$.

\begin{proposition}[Group-gradient magnitude under a two-cluster approximation]
\label{prop:two_cluster}
Assume there exist vectors $g^+,g^-$ such that $g_i(\theta)\approx g^+$ for samples with $r_i \ge \mu$
and $g_i(\theta)\approx g^-$ otherwise.
Let $S_+ := \{ i : r_i \ge \mu \}$ and define $S_+^{A}:=\sum_{i\in S_+}\hat A_i$.
Recalling that $\sum_{i=1}^G \hat A_i=0$ by construction, we have

\begin{equation}
    \lVert \nabla_\theta \mathcal{J}_{\mathrm{sur}}(\theta)\rVert
\;\approx\;
\frac{|S_+^{A}|}{G}\,\lVert g^+ - g^- \rVert.
\end{equation}
\end{proposition}

This proposition shows that the magnitude of the surrogate gradient scales with $|S_+^{A}|$. We provide the proof in Appendix~\ref{sec:proof1}.

\paragraph{High-variance regime.}

Here we show that $|S_+^{A}|$ is maximized when advantages exhibit substantial dispersion within the group.
Let $d_i=r_i-\mu$ (so $\sum d_i=0$), and recall that $\hat A_i=d_i/\sigma$.
Then the total positive advantage mass satisfies
\begin{equation}
    S_+^{A}
\;=\;
\sum_{i\in S_+}\hat A_i
\;\le\;
\frac{G}{2}.
\end{equation}
We provide the derivation in Appendix~\ref{sec:proof2}. 

The upper bound is tight for even $G$ when $|d_i|$ is constant across the group and exactly half of the $d_i$ are positive
(and half negative), i.e., $d_i\in\{+a,-a\}$ with equal counts.
In the binary reward case, this corresponds to a success rate of $p=0.5$, where reward variance $p(1-p)$ is maximized.

Combining with Proposition~\ref{prop:two_cluster}, the GRPO gradient is largest when group outcomes are maximally uncertain, i.e., when the within-group reward variance is highest.

\subsection{Pilot-Commit}
\label{sec:pilot_commit}

The analysis in Section~\ref{sec:grpo_high_variance} highlights the importance of prompts with high reward variance. In online RL, however, the reward distribution of a given prompt evolves as the policy updates, making it computationally prohibitive to roll out on the full prompt set at every training step. 

We therefore propose \emph{Pilot-Commit}, a sampling strategy that concentrates rollout budget on prompts exhibiting high reward variance under the current policy. Figure~\ref{fig:pc-overview} provides an overview of the algorithm, with full details given in Algorithm~\ref{alg:pc} in the Appendix. Although the framework applies to any scalar reward signal, we focus on the binary reward setting, where the success rate directly determines reward variance and provides a clean, interpretable proxy for prompt informativeness.

\paragraph{Pilot stage.}
Let $\mathcal{D}$ be the dataset of prompts and let $\mathcal{B} \subset \mathcal{D}$ be the sampling batch of size $b_g$.
For each prompt $q \in \mathcal{B}$, we sample $n_{\mathrm{pilot}}$ rollouts $o_j \sim \pi_\theta(\cdot \mid q)$ and compute rewards
$r_j = R(q,o_j)$, where $R(q,o)\in\{0,1\}$.
We estimate the success rate
\[
\hat p(q) = \frac{1}{n_{\mathrm{pilot}}}\sum_{j=1}^{n_{\mathrm{pilot}}} r_j.
\]
For binary rewards, the per-prompt variance satisfies $\mathrm{Var}[r\mid q] = p(1-p)$, so $\hat p(q)$ serves as a proxy for reward variance.
We keep prompts whose pilot success rate lies in a predefined range:
\[
\mathcal{B}' =
\left\{ q \in \mathcal{B} \mid
p_{\mathrm{lower}} \le \hat p(q) \le p_{\mathrm{upper}}
\right\}.
\]

Although $p(1-p)$ is symmetric around $\tfrac12$, we use asymmetric thresholds ($p_{\mathrm{lower}} < 1 - p_{\mathrm{upper}}$) to reflect the directional dynamics of training: prompts typically transition from unsolved ($p\approx 0$) through the high-variance regime ($p\approx 0.5$) toward solved ($p\approx 1$).
Hard prompts with rare correct rollouts are approaching the high-variance regime and should be retained, while near-solved prompts have already passed through it and their gradient signal is diminishing.

\paragraph{Commit stage.}
In the commit stage, we allocate the remaining rollouts only to prompts that passed the pilot filter. For each $q \in \mathcal{B}'$, we draw $n_{\mathrm{commit}}$ more samples
$\{o^{\mathrm{commit}}_j\}_{j=1}^{n_{\mathrm{commit}}}$ from $\pi_\theta(\cdot\mid q)$.
In practice, the number of retained prompts is capped at the training batch size $b_t$, where $b_t \leq b_g$. The replay buffer described below handles extra pilot survivors.

\paragraph{Training.}
We train on the union of pilot and commit rollouts for the selected prompts. For each retained prompt $q$, the training set includes
\[
\{(o_j, r_j)\}_{j=1}^{n_{\mathrm{pilot}}} \;\cup\; \{(o^{\mathrm{commit}}_j, r^{\mathrm{commit}}_j)\}_{j=1}^{n_{\mathrm{commit}}}.
\]

\subsection{Pilot-Commit Optimizations}
We introduce the following optimizations to make Pilot-Commit practical at scale.

\paragraph{Binding pilot and commit.}
A straightforward implementation would require two sequential inference passes per step (pilot then commit).
To reduce this overhead, we bind the pilot stage at step $t\!-\!1$ to the commit stage at step $t$, allowing the two passes to be overlapped.
This introduces an off-policy delay between the pilot and commit policies: one step from the binding itself, and up to $d$ steps for prompts drawn from the replay buffer. Prior work suggests small off-policy delays have limited impact in this regime~\citep{bartoldson2025trajectory}.
\paragraph{Replay buffer for pilot survivors.}
When the pilot stage yields more retained prompts than required, we store the excess prompts in a replay buffer rather than letting them go unused. Conversely, when the current pilot stage yields too few retained prompts, we sample additional prompts from the buffer to fill the batch. To limit staleness, we evict buffered prompts whose age exceeds $d$ steps.

\paragraph{Evicting solved prompts.}
As training progresses, prompts that initially have high uncertainty may become consistently solved. We provide empirical evidence in Section~\ref{sec:eviction}. To avoid repeatedly spending pilot rollouts on them, we evict prompts whose pilot success rate exceeds a threshold $p_{\mathrm{solve}}$ from the pilot sampling pool for subsequent epochs.

\section{Experimental Setup}

\subsection{Data and Models}

We evaluate Pilot-Commit across four model scales:
Qwen2.5-Math-1.5B~\citep{yang2024qwen2},
Qwen3-4B-Base,
Qwen3-8B-Base, and
Qwen3-14B-Base~\citep{yang2025qwen3}.
We use two training datasets depending on model scale.
For the 1.5B model, we train on DeepMath-103K~\citep{he2025deepmath}, which was extensively decontaminated against standard evaluation benchmarks. To further increase task difficulty, we remove binary (yes/no) questions, yielding approximately 85K training prompts.
For larger models (4B, 8B, 14B), we train on Polaris-53K~\citep{Polaris2025}. At these scales, DeepMath-103K becomes too easy: many prompts enter the near-solved regime early, causing reward saturation and training instability. Polaris-53K provides a more informative training distribution.

Evaluation is performed every 5 training steps on six mathematical reasoning benchmarks: AIME 2024, AIME 2025, AMC 2023, Math500, Minerva Math, and OlympiadBench. We additionally evaluate on a held-out DeepMath test set consisting of 500 prompts. For small benchmarks such as AIME and AMC, we report the mean accuracy over 32 independent trials to reduce variance.

\subsection{Training Details}

All experiments use a training batch size of 128 and a PPO minibatch size of 32. We train 1.5B for 1{,}000 steps and larger models for 500 steps, as they converge faster and showed training instability with extended training. Rewards are computed using the math verifier from Tinker-Cookbook~\citep{tml2025tinker}.

We do not apply an explicit KL penalty and use asymmetric clipping ratios of 0.2 (lower) and 0.28 (upper) following DAPO~\citep{yu2025dapo}. The maximum response length is 8{,}192 tokens for the 1.5B model and 4{,}096 for larger models. We use temperature 1.0 and $\texttt{top\_p}=1.0$ during training, and temperature 0.7 during evaluation. The learning rate is $1\times 10^{-6}$ with a 10-step warmup.

\paragraph{Pilot-Commit hyperparameters.}
For filtering based on pilot rollouts, we apply asymmetric thresholds. A prompt is considered too easy and skipped if its success rate is $> 0.75$. A prompt is considered too difficult and deferred if its success rate is $< 0.125$.

We allow a maximum off-policy delay of $d=4$ steps between pilot and commit. For eviction, we use a strict threshold of 1, evicting a prompt only when all pilot rollouts are correct.

\vspace{-5pt}
\subsection{Baselines and Sampling Cost}

We compare Pilot-Commit (PC) against two baselines: standard GRPO~\citep{shao2024deepseekmath} and DAPO~\citep{yu2025dapo}. DAPO oversamples prompts uniformly, assigns the same rollout budget to each, and filters out prompts with all-correct or all-incorrect rewards.

To compare efficiency, we analyze the sampling cost of each method. Let $n$ denote the rollout count per prompt for GRPO and DAPO. Recall that $b_t$ is the training batch size and $b_g = s \cdot b_t$ is the sampling batch size, where $s \geq 1$ is the oversampling factor. We set $s=3$ (i.e., $b_g = 3 \cdot b_t$), matching the oversampling factor used by DAPO so that both methods evaluate the same number of prompts per step at the pilot stage.

The sampling cost per step is:
\vspace{-2pt}
\begin{itemize}
\vspace{-2pt}
  \setlength{\itemsep}{0pt}
  \setlength{\topsep}{0pt}
  \setlength{\parskip}{0pt}
  \item \textbf{GRPO:} $b_t \times n$
  \item \textbf{DAPO:} $b_g \times n$
  \item \textbf{Pilot-Commit:} $b_g \times n_{\mathrm{pilot}} + b_t \times n_{\mathrm{commit}}$
  \vspace{-2pt}
\end{itemize}
\vspace{-2pt}

To match total training rollouts, we set $n = n_{\mathrm{pilot}} + n_{\mathrm{commit}}$, although PC can allow further saving by setting $n_{\mathrm{pilot}} + n_{\mathrm{commit}} < n$. In our experiments, we use $(n_{\mathrm{pilot}}, n_{\mathrm{commit}}) = (32, 96)$ for $n{=}128$, $(16, 48)$ for $n{=}64$, and $(8, 8)$ for $n{=}16$.
Overall, sampling cost increases in the order GRPO, Pilot-Commit, and DAPO.

\section{Results}

\subsection{Rollout Efficiency}
\label{sec:rollout_efficiency}

Our primary claim is that Pilot-Commit improves rollout efficiency: the rate of accuracy gained per generated rollout.
To evaluate this, we measure the cumulative rollout budget each method requires to reach a common accuracy target.
We set the target to the peak average accuracy achieved by GRPO in each setting, so that all three methods are compared against a level that the baseline is known to reach.
We report results across four model scales, covering both ample ($n{=}64$ or $128$) and limited ($n{=}16$) rollout budgets.

Table~\ref{table:rollout_efficiency} reports rollout-to-target results across all eight settings.
Under ample budgets ($n{=}64$ or $128$), PC reaches the target first in all four settings, using $1.5$--$1.9\times$ fewer rollouts than GRPO and $2.3$--$4.0\times$ fewer than DAPO.
Under limited budgets ($n{=}16$), PC still provides meaningful savings at 4B and 8B ($1.3$--$1.4\times$ fewer than GRPO, $2.3$--$3.3\times$ fewer than DAPO).
At 1.5B and 14B with $n{=}16$, PC and GRPO reach the target at comparable budgets (within 2\%), though PC achieves higher peak accuracy at 14B.

The larger gains under ample budgets are expected: when $n$ is large, each skipped prompt saves a full $n_{\mathrm{commit}}$ rollouts, so the absolute savings from filtering scale with the per-prompt budget. Under limited budgets, the pilot stage consumes a larger fraction of the total budget ($n_{\mathrm{pilot}} = 8$ out of $n = 16$), leaving less room for differential allocation. PC still avoids wasting commit rollouts on uninformative prompts, but the margin narrows.

\begin{table*}[t]
\caption{
Cumulative rollouts (in millions) to first reach the target average accuracy. The target is set to the peak accuracy of GRPO in each setting. \textbf{Bold} = fewest rollouts. ``*'' indicates the target is set to PC's peak when GRPO's peak is lower. ``never'' indicates the method did not reach the target within the training budget.
}
\label{table:rollout_efficiency}
\centering
\begin{tabular}{llccccccc}
\toprule
Config & Budget & Target & PC & GRPO & DAPO & GRPO/PC & DAPO/PC \\
\midrule
1.5B DeepMath & ample ($n{=}128$) & 44.0\% & \textbf{10.57M} & 15.97M & 42.02M & \textbf{1.5$\times$} & \textbf{4.0$\times$} \\
4B Polaris & ample ($n{=}64$) & 47.3\%* & \textbf{1.66M} & 2.50M & never & \textbf{1.5$\times$} & N/A \\
8B Polaris & ample ($n{=}64$) & 47.3\% & \textbf{1.04M} & 1.56M & 2.70M & \textbf{1.5$\times$} & \textbf{2.6$\times$} \\
14B Polaris & ample ($n{=}64$) & 57.0\%* & \textbf{2.15M} & 4.10M & 4.92M & \textbf{1.9$\times$} & \textbf{2.3$\times$} \\
\midrule
1.5B DeepMath & limited ($n{=}16$) & 41.3\% & \textbf{1.93M} & 1.97M & 4.49M & \textbf{1.0$\times$} & \textbf{2.3$\times$} \\
4B Polaris & limited ($n{=}16$) & 45.0\% & \textbf{0.51M} & 0.73M & 1.69M & \textbf{1.4$\times$} & \textbf{3.3$\times$} \\
8B Polaris & limited ($n{=}16$) & 47.1\% & \textbf{0.47M} & 0.63M & 1.26M & \textbf{1.3$\times$} & \textbf{2.7$\times$} \\
14B Polaris & limited ($n{=}16$) & 54.3\% & \textbf{0.51M} & 0.52M & never & 1.0$\times$ & N/A \\
\bottomrule
\end{tabular}
\end{table*}

\subsection{Quality of Each Policy Update}
\label{sec:update_quality}

Beyond reaching target accuracy faster, PC also produces more effective policy updates per training step.
We attribute this to the pilot-commit filtering, which concentrates each batch on prompts in the high reward-variance regime (Section~\ref{sec:grpo_high_variance}).

\paragraph{Reward variance.}
\figref{fig:update-quality}(a--b) plots the mean per-prompt reward standard deviation at each training step.
PC consistently maintains the highest reward variance throughout training, while GRPO's reward variance declines steadily as more prompts become saturated.
DAPO's oversampling partially mitigates this decline but still falls below PC.
This confirms that pilot-commit filtering concentrates updates on prompts with stronger gradient signals.

\paragraph{Peak accuracy.}
\figref{fig:update-quality}(c--d) shows average accuracy versus training steps.
Despite the same number of training steps, PC reaches higher peak accuracy than both GRPO and DAPO.
Combined with the reward variance analysis above, this suggests that each PC update receives a stronger gradient signal, improving not only the speed but also the quality of learning.

\begin{figure*}[!t]
    \centering
    \includegraphics[width=\textwidth]{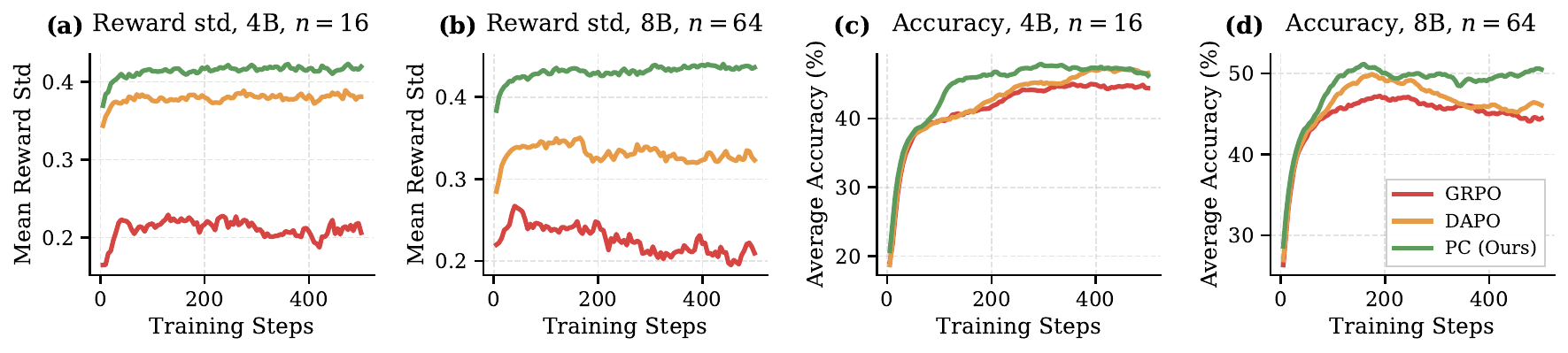}
    \caption{Quality of each policy update (all methods perform the same number of updates, one per step). (a--b) Mean per-prompt reward standard deviation over training steps. (c--d) Average accuracy versus training steps. Per-benchmark breakdowns are in the Appendix.}
    \label{fig:update-quality}
\end{figure*}

\subsection{Wall-Clock Runtime}
\label{sec:runtime}

Rollout count is our primary efficiency metric, but we also report wall-clock runtime to confirm that rollout savings translate into practical cost reductions.
Table~\ref{table:runtime} reports per-step runtime on Qwen3-14B ($n{=}64$, 64 H200 GPUs with TP=1), a highly parallelized setup where rollout generation is distributed across 64 inference engines.

Since all methods share the same training batch size, the policy update time is roughly constant across methods (${\sim}104$--$115$s).
Rollout generation is therefore the only variable cost component, and reducing rollout count directly reduces wall-clock time.
Even in this parallelism-heavy regime where rollout time is already well-amortized, the wall-clock difference remains substantial (1.21$\times$ for PC vs 1.83$\times$ for DAPO). With fewer inference engines or larger models requiring tensor parallelism, rollout generation would constitute a larger fraction of step time, further favoring rollout-efficient methods.

\begin{table*}[t]
\caption{Per-step wall-clock runtime on Qwen3-14B ($n{=}64$, 8 nodes, 64 H200 GPUs). All times are in seconds, averaged over 500 steps (excluding evaluation steps for total time).}
\label{table:runtime}
\centering
\resizebox{\textwidth}{!}{
\begin{tabular}{lc|ccc|cc}
\toprule
Method & Rollouts/step & Rollout time & Update time & Total time & Rollout slowdown & Total slowdown \\
\midrule
GRPO & 8{,}192 & 106 & 104 & 248 & 1.00$\times$ & 1.00$\times$ \\
PC & 12{,}288 & 138 & 115 & 299 & 1.30$\times$ & 1.21$\times$ \\
DAPO & 24{,}576 & 280 & 110 & 453 & 2.64$\times$ & 1.83$\times$ \\
\bottomrule
\end{tabular}}
\end{table*}

\subsection{Final Accuracy under Fixed Training Steps}
\label{sec:fixed_step}

For completeness, we report peak per-benchmark accuracy under a fixed number of training steps in Tables~\ref{table:ample_budget} and~\ref{table:limited_budget} in the Appendix.
Note that this comparison does not equalize compute: under ample budgets each PC step generates $1.5\times$ the rollouts of GRPO ($2\times$ under limited budgets), and DAPO generates $3\times$, so matching steps implicitly grants more rollout budget to the costlier methods.
Even under this comparison, PC achieves the highest or comparable peak accuracy in 7 of 8 settings, while consuming fewer total rollouts than DAPO in every case.

\section{Analysis}

\subsection{Tradeoff Between Pilot and Commit Allocation}
\label{sec:budget_alloc}

We sweep over different pilot/commit splits $(n_{\mathrm{pilot}}, n_{\mathrm{commit}})$ to understand which component of rollout cost drives performance: the total number of rollouts generated (sampling cost), or the number used for training.

\paragraph{Equal training cost.}
We fix $n_{\mathrm{pilot}} + n_{\mathrm{commit}} = 64$, so every configuration trains on the same number of rollouts per prompt, but pilot-heavy splits generate more total rollouts due to the $3\times$ sampling batch (Figure~\ref{fig:alloc_train_fixed_mean}).
Lightweight pilot allocations such as $(8,56)$ and $(16,48)$ achieve the best rollout efficiency, reaching a given accuracy with fewer cumulative rollouts.
We adopt $(16,48)$ in our main experiments; although $(8,56)$ is slightly more efficient, it is more prone to premature eviction due to limited pilot evidence.

\paragraph{Equal sampling cost.}
We hold the total rollouts per step constant at $12{,}288$ while varying the split (Figure~\ref{fig:alloc_sample_fixed_mean}).
Despite large differences in the number of training rollouts per update (ranging from ${\sim}40$ to ${\sim}80$), all configurations converge to similar final accuracy.
This confirms that total sampling cost, not the number of training rollouts, is the primary factor governing performance, and that diverting a portion of the rollout budget to prompt selection incurs no accuracy penalty.

\begin{figure}[!t]
    \centering
    \begin{subfigure}[t]{0.48\columnwidth}
        \centering
        \includegraphics[width=\linewidth]{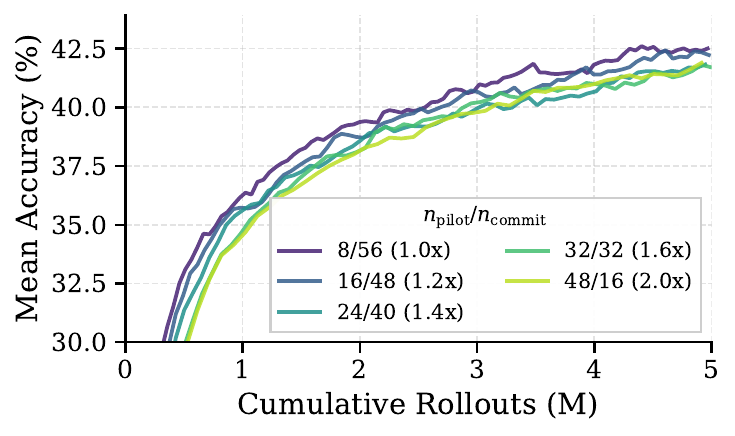}
        \caption{Equal training cost.}
        \label{fig:alloc_train_fixed_mean}
    \end{subfigure}
    \hfill
    \begin{subfigure}[t]{0.48\columnwidth}
        \centering
        \includegraphics[width=\linewidth]{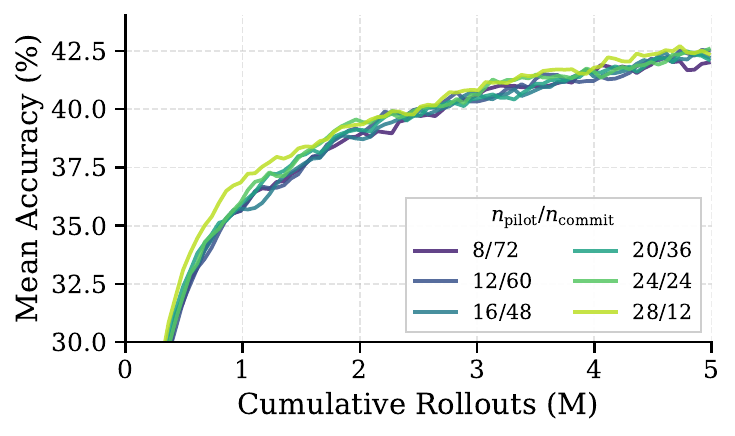}
        \caption{Equal sampling cost.}
        \label{fig:alloc_sample_fixed_mean}
    \end{subfigure}
    \caption{Pilot/commit allocation sweep on 1.5B DeepMath. (a) Equal training cost: training rollouts per prompt are fixed ($n_{\mathrm{pilot}} + n_{\mathrm{commit}} = 64$) while varying the split. Parenthetical values show relative sampling cost per step (normalized to the cheapest configuration). (b) Equal sampling cost: total rollouts per step are held constant at $12{,}288$ while varying the pilot/commit split; the number of training rollouts per update differs.}
    \label{fig:budget_alloc}
\end{figure}

\subsection{Hyperparameter Robustness}
\label{sec:hp_robustness}

Pilot-Commit introduces several hyperparameters beyond standard GRPO: the pilot thresholds $p_{\mathrm{lower}}$ and $p_{\mathrm{upper}}$, and the eviction threshold $p_{\mathrm{solve}}$.
We conduct ablations on a Qwen2.5-Math-1.5B model trained for 1{,}000 steps on Polaris-53K.
Overall, Pilot-Commit is robust to these hyperparameters across a wide range of values, though optimal thresholds may vary with the difficulty distribution of the training data.

\paragraph{Upper threshold ($p_{\mathrm{upper}}$).}
We sweep $p_{\mathrm{upper}} \in \{1.0, 0.875, 0.75, 0.625\}$ (Figure~\ref{fig:hp_ablation}a).
All four settings, including no filtering ($p_{\mathrm{upper}}{=}1.0$), converge to similar final accuracy, indicating that performance is robust to this threshold.
We default to $0.75$ as a conservative choice.

\paragraph{Lower threshold ($p_{\mathrm{lower}}$).}
We sweep $p_{\mathrm{lower}} \in \{0.0, 0.125, 0.25, 0.375\}$ (Figure~\ref{fig:hp_ablation}b).
Both extremes underperform slightly: no filtering ($p_{\mathrm{lower}}{=}0.0$) wastes budget on hopeless prompts, while aggressive filtering ($p_{\mathrm{lower}}{=}0.375$) discards borderline prompts that still provide useful signal.
Our default of $0.125$ achieves the best accuracy, balancing these two effects.

\paragraph{Eviction threshold ($p_{\mathrm{solve}}$).}
We compare $p_{\mathrm{solve}} \in \{0.75, 0.875, 1.0\}$ (Figure~\ref{fig:hp_ablation}c).
All three settings converge to similar final accuracy, despite substantially different eviction rates: $p_{\mathrm{solve}}{=}0.75$ evicts ${\sim}35\%$ of prompts vs ${\sim}22\%$ for $1.0$.
We default to $p_{\mathrm{solve}}{=}1.0$, the most conservative setting, to minimize the risk of premature eviction.

\begin{figure*}[!t]
    \centering
    \includegraphics[width=\textwidth]{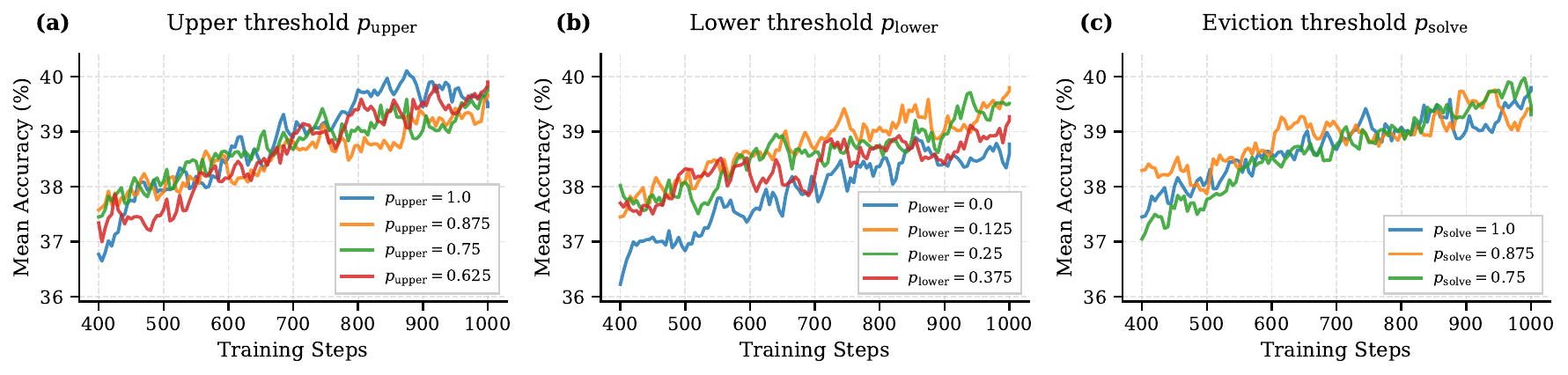}
    \caption{Hyperparameter ablations on 1.5B Polaris. (a) Upper threshold $p_{\mathrm{upper}}$: controls filtering of near-solved prompts ($1.0$ = no filtering). (b) Lower threshold $p_{\mathrm{lower}}$: controls filtering of hard prompts ($0.0$ = no filtering). (c) Eviction threshold $p_{\mathrm{solve}}$: controls when solved prompts are permanently removed from the dataset.}
    \label{fig:hp_ablation}
\end{figure*}

\subsection{Impact of the Eviction Strategy}
\label{sec:eviction}

Pilot-Commit's pilot batch is $3\times$ the training batch, so PC cycles through prompts much faster than GRPO and encounters already-solved prompts within a few hundred steps.
Eviction permanently removes these prompts to avoid wasting sampling budget.

\paragraph{Solved prompts are stable.}
We verify that solved prompts rarely revert: using GRPO (1.5B) trained for three epochs, 81.2\% of prompts fully solved in Epoch~2 remain solved in Epoch~3, justifying permanent removal.

\paragraph{Eviction reduces cost without hurting accuracy.}
Table~\ref{table:eviction} compares PC with and without eviction.
Without eviction, 35.8\% of training steps require a second sampling round to fill the training batch, raising the mean rollouts per step to 33{,}379.
With eviction, only 4.5\% of steps need a second round, reducing mean rollouts to 25{,}676 (a 23\% saving) while peak accuracy is unchanged (46.8\% vs 46.9\%).

\begin{table}[t]
\centering
\small
\setlength{\tabcolsep}{4pt}
\begin{tabular}{@{}lcc@{}}
\toprule
Method & Acc.\ (\%) & Rollouts / Step \\
\midrule
PC (no eviction)  & \textbf{46.9} & 33{,}379 \\
PC (with eviction) & 46.8         & \textbf{25{,}676} \\
\bottomrule
\end{tabular}
\caption{Effect of eviction (1.5B DeepMath, $n{=}128$).}
\label{table:eviction}
\end{table}

\subsection{Data Efficiency of Pilot-Commit}

By design, PC's pilot stage screens more prompts per step than end up in the training batch, so PC traverses the dataset faster than GRPO for the same number of optimization steps.
Figure~\ref{fig:data-eff} plots accuracy against total prompts consumed: GRPO achieves higher accuracy per prompt, followed by PC and DAPO.
This tradeoff is inherent to screening-based allocation: pilot rollouts estimate per-prompt informativeness at the cost of faster dataset traversal.
Crucially, RL post-training is usually multi-epoch: prompts cycle back across epochs, so prompts deferred by the pilot are not permanently lost but re-encountered in later passes.
Over 1{,}000 steps (${\sim}85$K prompts in the dataset), GRPO completes ${\sim}1.5$ epochs; PC screens the same number of prompts as DAPO (${\sim}4.5$ epochs), but eviction of solved prompts shrinks the active pool, so PC effectively cycles through ${\sim}6.2$ epochs, giving each prompt multiple chances to pass the pilot filter.
In our setting, rollout generation dominates wall-clock time, so we optimize for cumulative rollouts rather than prompt reuse.

\begin{figure}[t]
    \centering
    \includegraphics[width=0.5\columnwidth]{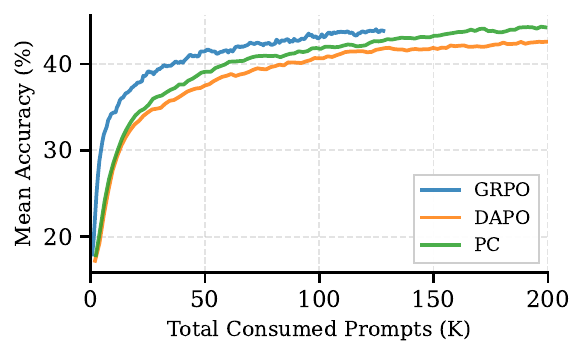}
    \caption{Accuracy vs.\ total consumed prompts (1.5B DeepMath, $n{=}128$); horizontal axis truncated at $200{\mathrm{k}}$ prompts.}
    \label{fig:data-eff}
\end{figure}
\section{Limitations \& Future Work}
Pilot-Commit currently assumes verifiable, binary rewards, where the success rate directly determines reward variance.
Extending the pilot estimator to continuous or noisy reward signals, such as RLHF with learned reward models, is the most immediate open problem.
Our filtering thresholds ($p_{\mathrm{lower}}, p_{\mathrm{upper}}$) are also static throughout training.
Early in training, when most prompts have near-zero success rates, more permissive thresholds may be appropriate; adapting them to the evolving difficulty distribution could further improve efficiency.

\section{Conclusion}
Group-based RL post-training spends the majority of its compute on rollout generation, yet existing methods allocate rollouts uniformly across prompts regardless of learning signal.
Pilot-Commit addresses this by splitting the rollout budget into a pilot stage that estimates per-prompt informativeness and a commit stage that concentrates the remaining rollouts on prompts in the high reward-variance regime where policy gradients are strongest.
Across model scales from 1.5B to 14B and multiple math reasoning benchmarks, this strategy matches baseline accuracy while reaching it up to $1.9\times$ faster than GRPO and $4.0\times$ faster than DAPO in cumulative rollouts.

\clearpage

\bibliographystyle{assets/plainnat}
\bibliography{ref}

%%%%%%%%%%%%%%%%%%%%%%%%%%%%%%%%%%%%%%%%%%%%%%%%%%%%%%%%%%%%%%%%%%%%%%%%%%%%%%%
%%%%%%%%%%%%%%%%%%%%%%%%%%%%%%%%%%%%%%%%%%%%%%%%%%%%%%%%%%%%%%%%%%%%%%%%%%%%%%%
% APPENDIX
%%%%%%%%%%%%%%%%%%%%%%%%%%%%%%%%%%%%%%%%%%%%%%%%%%%%%%%%%%%%%%%%%%%%%%%%%%%%%%%
%%%%%%%%%%%%%%%%%%%%%%%%%%%%%%%%%%%%%%%%%%%%%%%%%%%%%%%%%%%%%%%%%%%%%%%%%%%%%%%
\newpage
\appendix
\onecolumn
\section{Appendix}

\subsection{Pilot-Commit Algorithm}

\begin{algorithm}[tb]
\caption{Pilot-Commit}
\label{alg:pc}
\begin{algorithmic}[1]

\Require $n_{\mathrm{pilot}}$, $n_{\mathrm{commit}}$, thresholds $p_{\mathrm{lower}}, p_{\mathrm{upper}}, p_{\mathrm{solve}}$, max delay $d$, dataset $\mathcal{D}$, sampling batch size $b_g$, training batch size $b_t$ ($b_g > b_t$)

\State $\mathcal{R} \gets \emptyset$ \Comment{Replay buffer}
\State $t \gets 0$

\For{each epoch}
  \For{each sampling batch $\mathcal{B} \subset \mathcal{D}$, $|\mathcal{B}| = b_g$}

    \State \textbf{Pilot:} Sample $n_{\mathrm{pilot}}$ rollouts per prompt $q \in \mathcal{B}$; compute $\hat p(q)$
    \State $\mathcal{B}' \gets \{ q \in \mathcal{B} \mid p_{\mathrm{lower}} \le \hat p(q) \le p_{\mathrm{upper}} \}$ \Comment{Filter}
    \State Remove $\{ q \in \mathcal{B} \mid \hat p(q) \ge p_{\mathrm{solve}} \}$ from $\mathcal{D}$ \Comment{Evict solved prompts}
    \State $\mathcal{R} \gets \mathcal{R} \cup \mathcal{B}'$ \Comment{Buffer survivors}

    \State \textbf{Commit:} \Comment{Prompts in $\mathcal{R}$ may come from previous steps}
    \While{fewer than $b_t$ prompts selected}
      \State Sample prompt $q$ from $\mathcal{R}$
      \State Draw $n_{\mathrm{commit}}$ rollouts from current policy $\pi_\theta(\cdot \mid q)$
    \EndWhile
    \State Train on union of stored pilot + new commit rollouts for $b_t$ selected prompts
    \State Remove entries from $\mathcal{R}$ older than $d$ steps
    \State $t \gets t + 1$

  \EndFor
\EndFor

\end{algorithmic}
\end{algorithm}

\subsection{Proof of Proposition~\ref{prop:two_cluster}}
\label{sec:proof1}

\begin{proof}
Fix a prompt $q$ and a sampled group $\{o_i\}_{i=1}^G$. Recall
\[
\nabla_\theta \mathcal{J}_{\mathrm{sur}}(\theta)
=
\frac{1}{G}\sum_{i=1}^G \hat{A}_i\, g_i(\theta).
\]
Let $S_+ := \{ i : r_i \ge \mu \}$ and $S_- := \{1,\dots,G\}\setminus S_+$. Under the two-cluster approximation,
\[
g_i(\theta) \approx g^+ \ \text{for } i\in S_+,
\qquad
g_i(\theta) \approx g^- \ \text{for } i\in S_-.
\]
Therefore,
\begin{align*}
\nabla_\theta \mathcal{J}_{\mathrm{sur}}(\theta)
&\approx
\frac{1}{G}\left(
\sum_{i\in S_+}\hat{A}_i\, g^+
+
\sum_{i\in S_-}\hat{A}_i\, g^-
\right) \\
&=
\frac{1}{G}\left(
S_+^{A}\, g^+
+
S_-^{A}\, g^-
\right),
\end{align*}
where $S_+^{A}:=\sum_{i\in S_+}\hat A_i$ and $S_-^{A}:=\sum_{i\in S_-}\hat A_i$.

By construction of group-relative normalization,
\[
\sum_{i=1}^G \hat A_i = 0
\quad\Longrightarrow\quad
S_-^{A} = -S_+^{A}.
\]
Substituting gives
\begin{align*}
\nabla_\theta \mathcal{J}_{\mathrm{sur}}(\theta)
&\approx
\frac{1}{G}\left(
S_+^{A}\, g^+ - S_+^{A}\, g^-
\right)
=
\frac{S_+^{A}}{G}\,(g^+ - g^-).
\end{align*}
Taking norms yields
\[
\left\lVert \nabla_\theta \mathcal{J}_{\mathrm{sur}}(\theta)\right\rVert
\approx
\frac{|S_+^{A}|}{G}\,\lVert g^+ - g^- \rVert,
\]
which is the claimed relation.
\end{proof}

\subsection{Derivation of the bound in the high-variance regime}
\label{sec:proof2}

\begin{lemma}[Positive advantage mass bound]
\label{lem:pos_mass_bound}
Let $d_i := r_i - \mu$ so that $\sum_{i=1}^G d_i = 0$, and let
\[
\sigma := \sqrt{\frac{1}{G}\sum_{i=1}^G d_i^2},
\qquad
\hat A_i := \frac{d_i}{\sigma}.
\]
Define $S_+ := \{i : d_i \ge 0\}$ and $S_+^{A} := \sum_{i\in S_+} \hat A_i$. Then
\[
S_+^{A} \le \frac{G}{2}.
\]
Moreover, equality holds when $|d_i|$ is constant across $i$ and exactly half of the $d_i$ are positive and half are negative (hence $G$ is even).
\end{lemma}

\begin{proof}
By definition,
\[
S_+^{A} = \sum_{i\in S_+}\hat A_i = \frac{1}{\sigma}\sum_{i\in S_+} d_i.
\]
Since $\sum_{i=1}^G d_i = 0$, the total positive mass equals the total negative magnitude:
\[
\sum_{i\in S_+} d_i
=
-\sum_{i\in S_-} d_i
=
\sum_{i\in S_-} |d_i|.
\]
Consequently,
\[
\sum_{i\in S_+} d_i
=
\frac{1}{2}\sum_{i=1}^G |d_i|,
\]
and therefore
\begin{equation}
\label{eq:SplusA_l1}
S_+^{A}
=
\frac{1}{2\sigma}\sum_{i=1}^G |d_i|.
\end{equation}

Next apply Cauchy--Schwarz to the vectors $(|d_1|,\dots,|d_G|)$ and $(1,\dots,1)$:
\[
\sum_{i=1}^G |d_i|
\le
\sqrt{\sum_{i=1}^G 1^2}\,\sqrt{\sum_{i=1}^G d_i^2}
=
\sqrt{G}\,\sqrt{\sum_{i=1}^G d_i^2}.
\]
Using $\sigma = \sqrt{\frac{1}{G}\sum_{i=1}^G d_i^2}$, we have $\sqrt{\sum_{i=1}^G d_i^2} = \sqrt{G}\,\sigma$, hence
\[
\sum_{i=1}^G |d_i| \le \sqrt{G}\,(\sqrt{G}\,\sigma) = G\sigma.
\]
Plugging this into \eqref{eq:SplusA_l1} gives
\[
S_+^{A}
=
\frac{1}{2\sigma}\sum_{i=1}^G |d_i|
\le
\frac{1}{2\sigma}\,(G\sigma)
=
\frac{G}{2},
\]
as claimed.

For tightness, equality in Cauchy--Schwarz holds iff $|d_i|$ is proportional to $1$ for all $i$, that is, $|d_i| \equiv a$ for some $a\ge 0$.
Equality in $\sum_{i\in S_+} d_i = \frac{1}{2}\sum_i |d_i|$ requires that the positive and negative masses match, which with $|d_i|\equiv a$ occurs exactly when the number of $+a$ and $-a$ entries are equal. This requires $G$ even and yields $d_i\in\{+a,-a\}$ with equal counts, proving the stated condition for tightness.
\end{proof}

\subsection{Peak Accuracy under Fixed Training Steps}

\begin{table}[H]
\caption{
Peak accuracy under fixed training steps with ample rollout budgets: $n{=}128$ (1.5B) and $n{=}64$ (4B--14B). ``Rollouts Trained'' is the number of rollouts used for the policy update per step; ``Rollouts Sampled'' is the total number of rollouts generated per step.}
\label{table:ample_budget}
\centering
\resizebox{\textwidth}{!}{
\begin{tabular}{lcccccccc}
\toprule
Method & AIME 24 & AIME 25 & MATH 500 & Minerva Math & Olympiad Bench & AMC 23 & DeepMath & Rollouts Trained / Sampled \\
\midrule
\midrule
1.5B-GRPO & 18.6 & 17.5 & 80.8 & \textbf{29.8} & 44.1 & 60.9 & 69.4 & 16{,}384 / 16{,}384 \\
1.5B-DAPO & 18.6 & 17.0 & 81.8 & \textbf{29.8} & \textbf{46.3} & 62.8 & 71.0 & 16{,}384 / 49{,}152 \\
1.5B-PC & \textbf{20.2} & \textbf{19.5} & \textbf{83.2} & 29.0 & 45.3 & \textbf{65.0} & \textbf{72.6} & 16{,}384 / 24{,}576 \\
\midrule
4B-GRPO & \textbf{24.7} & 19.7 & 83.2 & 34.9 & 48.4 & \textbf{69.3} & 66.2 & 8{,}192 / 8{,}192 \\
4B-DAPO & 23.2 & 19.2 & 83.0 & 35.3 & 46.9 & 67.0 & 65.6 & 8{,}192 / 24{,}576 \\
4B-PC & 23.4 & \textbf{20.3} & \textbf{84.8} & \textbf{36.8} & \textbf{49.9} & 67.0 & \textbf{67.0} & 8{,}192 / 12{,}288 \\
\midrule
8B-GRPO & 27.4 & 16.4 & 82.8 & 36.0 & 49.0 & 70.2 & 64.2 & 8{,}192 / 8{,}192 \\
8B-DAPO & \textbf{29.3} & 20.0 & \textbf{87.0} & 37.1 & 51.2 & 70.4 & 70.2 & 8{,}192 / 24{,}576 \\
8B-PC & 27.3 & \textbf{20.9} & 86.4 & \textbf{39.0} & \textbf{52.1} & \textbf{75.8} & \textbf{71.6} & 8{,}192 / 12{,}288 \\
\midrule
14B-GRPO & 40.1 & 28.0 & 89.6 & \textbf{39.7} & 58.5 & 84.1 & 77.4 & 8{,}192 / 8{,}192 \\
14B-DAPO & \textbf{40.2} & \textbf{28.6} & \textbf{91.0} & 39.3 & \textbf{59.3} & 84.5 & \textbf{78.0} & 8{,}192 / 24{,}576 \\
14B-PC & 39.9 & 25.5 & 90.0 & 39.0 & 57.3 & \textbf{85.0} & 76.6 & 8{,}192 / 12{,}288 \\
\bottomrule
\end{tabular}}
\end{table}

\begin{table}[H]
\caption{
Peak accuracy under fixed training steps with limited rollout budgets ($n{=}16$). ``Rollouts Trained'' is the number of rollouts used for the policy update per step; ``Rollouts Sampled'' is the total number of rollouts generated per step.}
\label{table:limited_budget}
\centering
\resizebox{\textwidth}{!}{
\begin{tabular}{lcccccccc}
\toprule
Method & AIME 24 & AIME 25 & MATH 500 & Minerva Math & Olympiad Bench & AMC 23 & DeepMath & Rollouts Trained / Sampled \\
\midrule
\midrule
1.5B-GRPO & 16.6 & 14.3 & \textbf{78.8} & 28.7 & 41.8 & 57.8 & 64.0 & 2{,}048 / 2{,}048 \\
1.5B-DAPO & 17.0 & 14.5 & \textbf{78.8} & \textbf{29.4} & \textbf{43.0} & 59.5 & 65.0 & 2{,}048 / 6{,}144 \\
1.5B-PC & \textbf{17.1} & \textbf{15.3} & 77.8 & 28.7 & \textbf{43.0} & \textbf{61.7} & \textbf{65.2} & 2{,}048 / 4{,}096 \\
\midrule
4B-GRPO & 20.4 & 17.4 & 81.2 & 35.3 & 47.0 & 65.5 & 62.2 & 2{,}048 / 2{,}048 \\
4B-DAPO & 23.5 & 19.1 & 82.4 & 34.6 & 47.6 & \textbf{69.5} & \textbf{67.0} & 2{,}048 / 6{,}144 \\
4B-PC & \textbf{24.3} & \textbf{19.6} & \textbf{85.4} & \textbf{36.8} & \textbf{49.7} & 69.4 & 65.8 & 2{,}048 / 4{,}096 \\
\midrule
8B-GRPO & \textbf{27.6} & 16.6 & 83.0 & 38.2 & 45.7 & 67.7 & 66.4 & 2{,}048 / 2{,}048 \\
8B-DAPO & 26.5 & 20.8 & \textbf{87.2} & 37.5 & \textbf{52.5} & 72.8 & \textbf{69.6} & 2{,}048 / 6{,}144 \\
8B-PC & 26.2 & \textbf{21.4} & 87.0 & \textbf{39.0} & 50.7 & \textbf{73.5} & 68.0 & 2{,}048 / 4{,}096 \\
\midrule
14B-GRPO & 31.1 & 23.6 & 89.0 & 38.6 & 56.7 & 78.7 & 74.4 & 2{,}048 / 2{,}048 \\
14B-DAPO & \textbf{37.3} & 21.9 & 88.6 & 38.2 & 52.7 & 75.2 & 73.4 & 2{,}048 / 6{,}144 \\
14B-PC & 37.1 & \textbf{27.5} & \textbf{90.2} & \textbf{39.3} & \textbf{57.0} & \textbf{84.2} & \textbf{75.6} & 2{,}048 / 4{,}096 \\
\bottomrule
\end{tabular}}
\end{table}

\subsection{Per-Benchmark Training Curves}

The following figures show per-benchmark accuracy for all model scales and rollout budgets. Each figure contains 8 subplots (7 individual benchmarks + average). Figures are presented in both step-normalized (accuracy vs training steps) and rollout-normalized (accuracy vs cumulative rollouts) formats.

\subsubsection{Qwen2.5-Math-1.5B (DeepMath-103K)}

\begin{figure}[H]
    \centering
    \includegraphics[width=\textwidth]{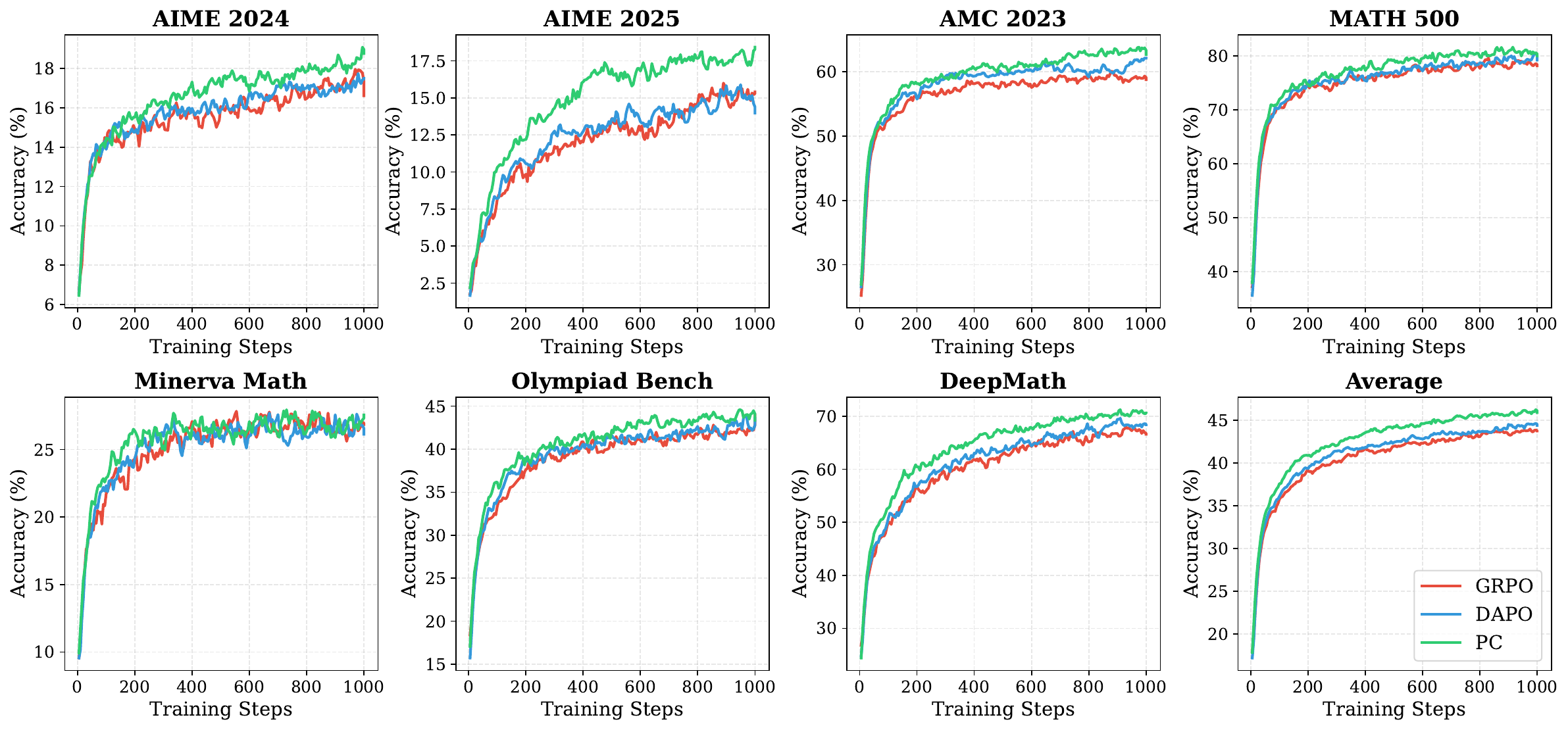}
    \caption{Qwen2.5-Math-1.5B, $n{=}128$: accuracy vs training steps.}
    \label{fig:app-1.5b-128-steps}
\end{figure}

\begin{figure}[H]
    \centering
    \includegraphics[width=\textwidth]{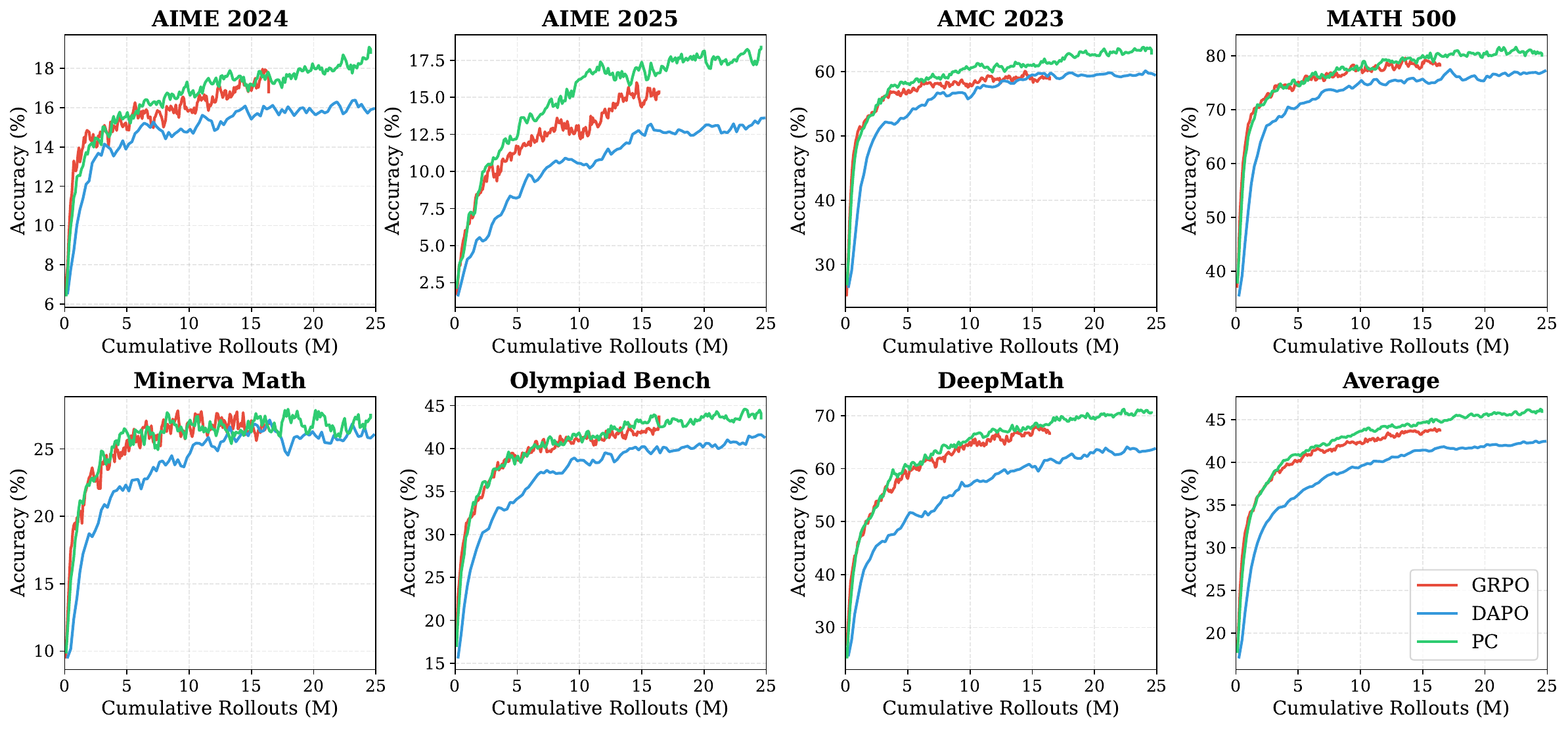}
    \caption{Qwen2.5-Math-1.5B, $n{=}128$: accuracy vs cumulative rollouts.}
    \label{fig:app-1.5b-128-rollouts}
\end{figure}

\begin{figure}[H]
    \centering
    \includegraphics[width=\textwidth]{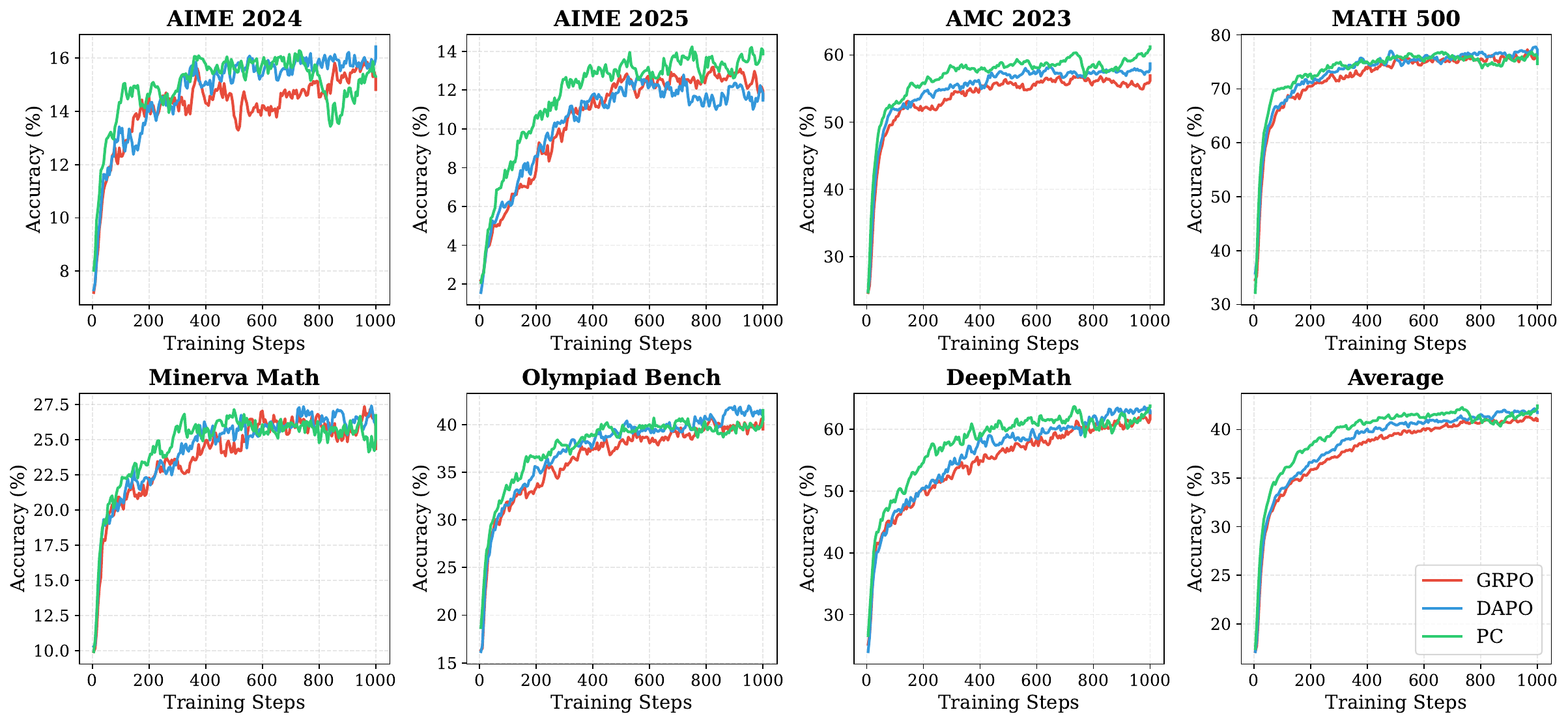}
    \caption{Qwen2.5-Math-1.5B, $n{=}16$: accuracy vs training steps.}
    \label{fig:app-1.5b-16-steps}
\end{figure}

\begin{figure}[H]
    \centering
    \includegraphics[width=\textwidth]{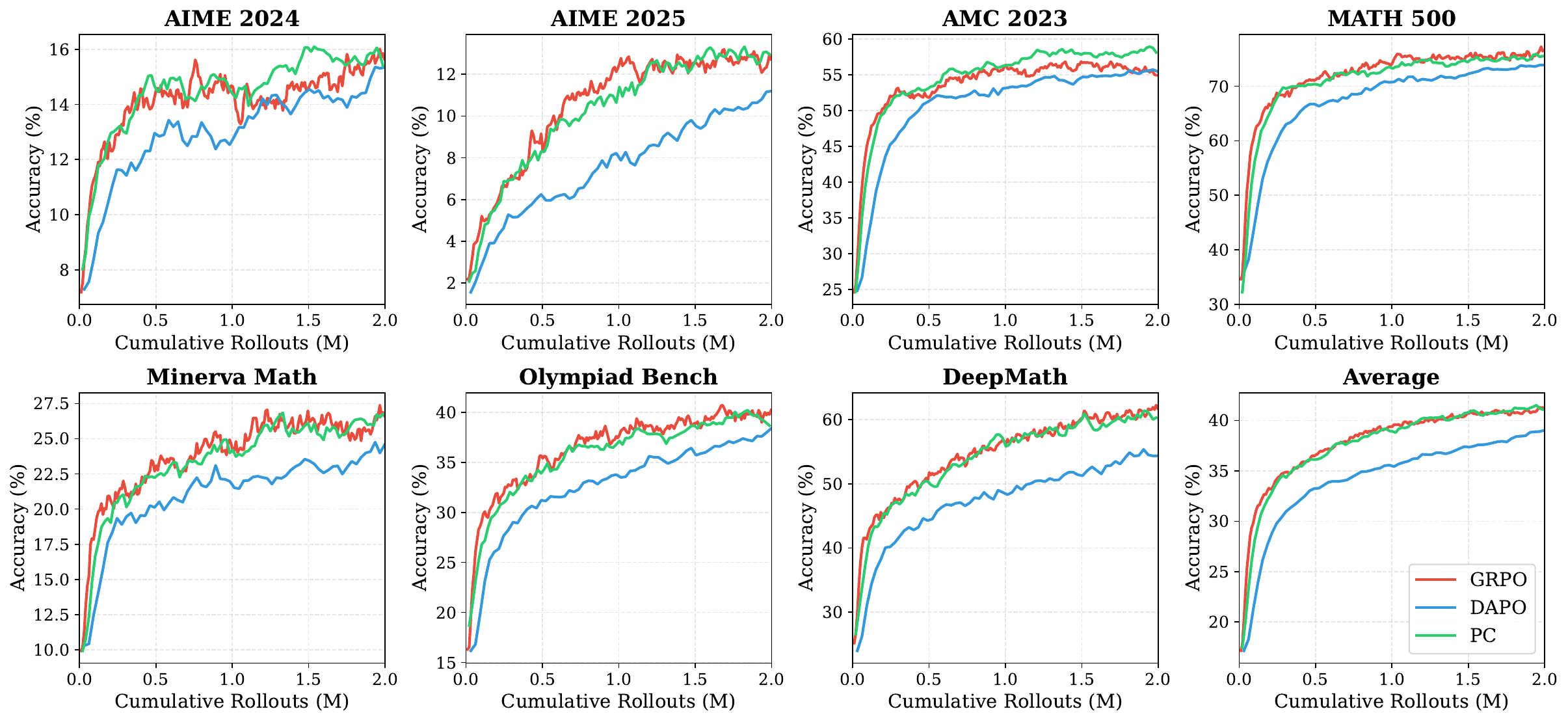}
    \caption{Qwen2.5-Math-1.5B, $n{=}16$: accuracy vs cumulative rollouts.}
    \label{fig:app-1.5b-16-rollouts}
\end{figure}

\subsubsection{Qwen3-4B (Polaris-53K)}

\begin{figure}[H]
    \centering
    \includegraphics[width=\textwidth]{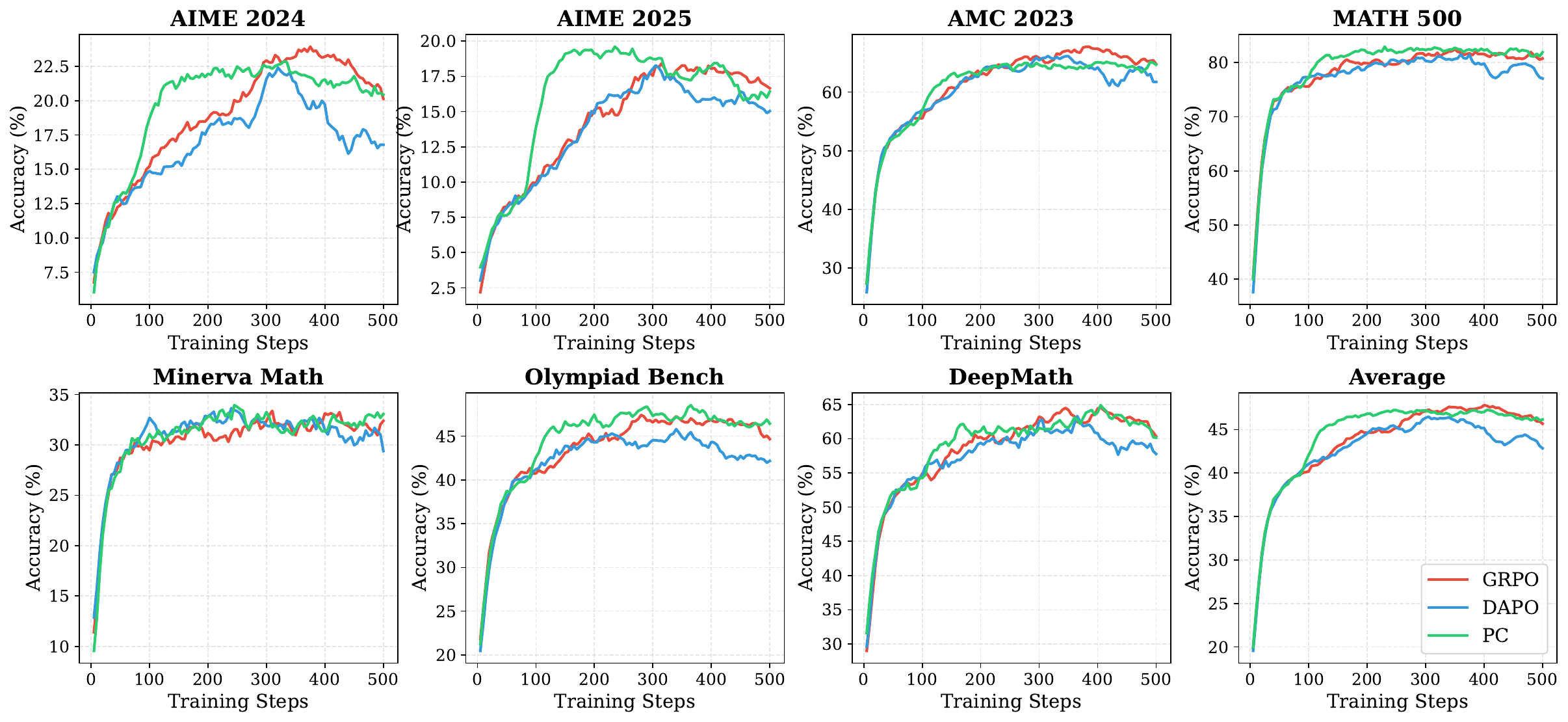}
    \caption{Qwen3-4B, $n{=}64$: accuracy vs training steps.}
    \label{fig:app-4b-64-steps}
\end{figure}

\begin{figure}[H]
    \centering
    \includegraphics[width=\textwidth]{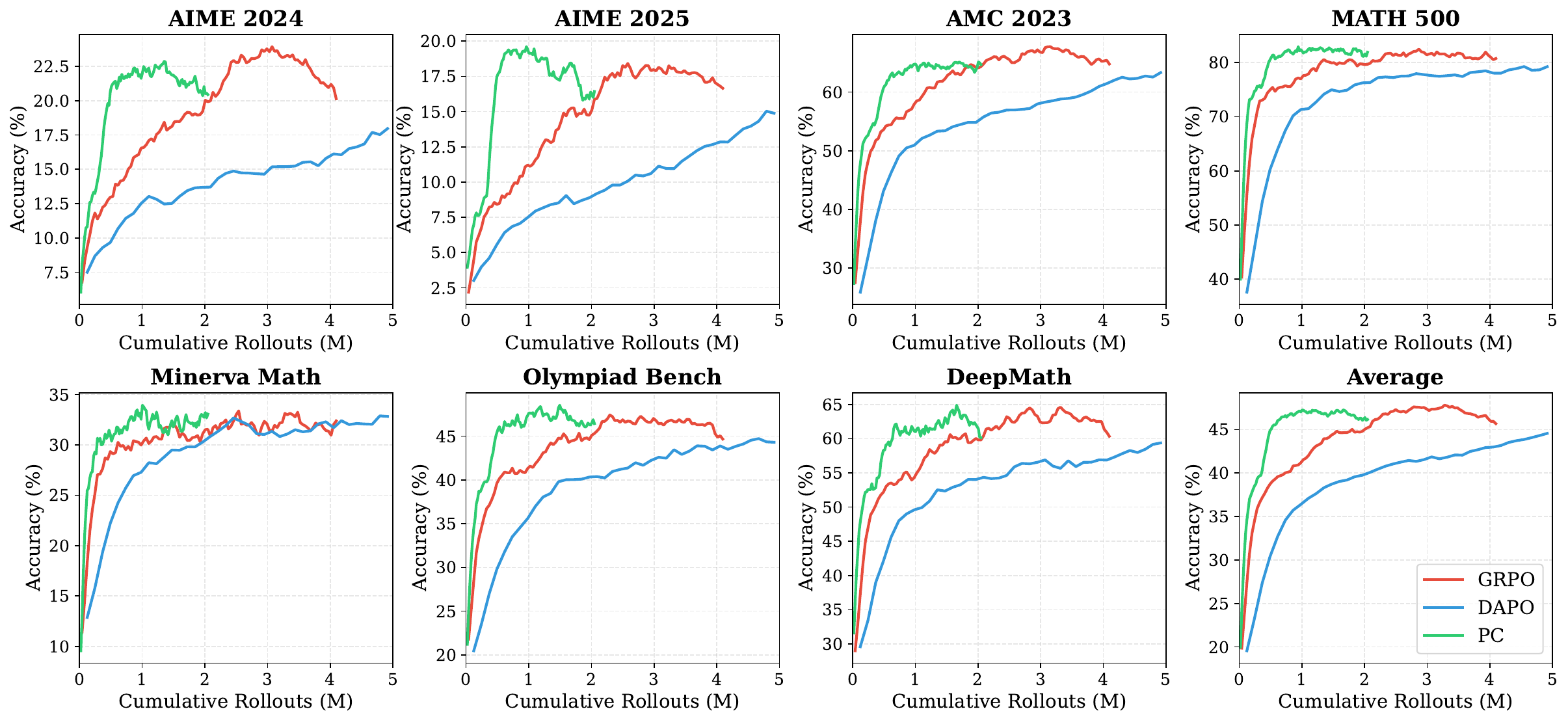}
    \caption{Qwen3-4B, $n{=}64$: accuracy vs cumulative rollouts.}
    \label{fig:app-4b-64-rollouts}
\end{figure}

\begin{figure}[H]
    \centering
    \includegraphics[width=\textwidth]{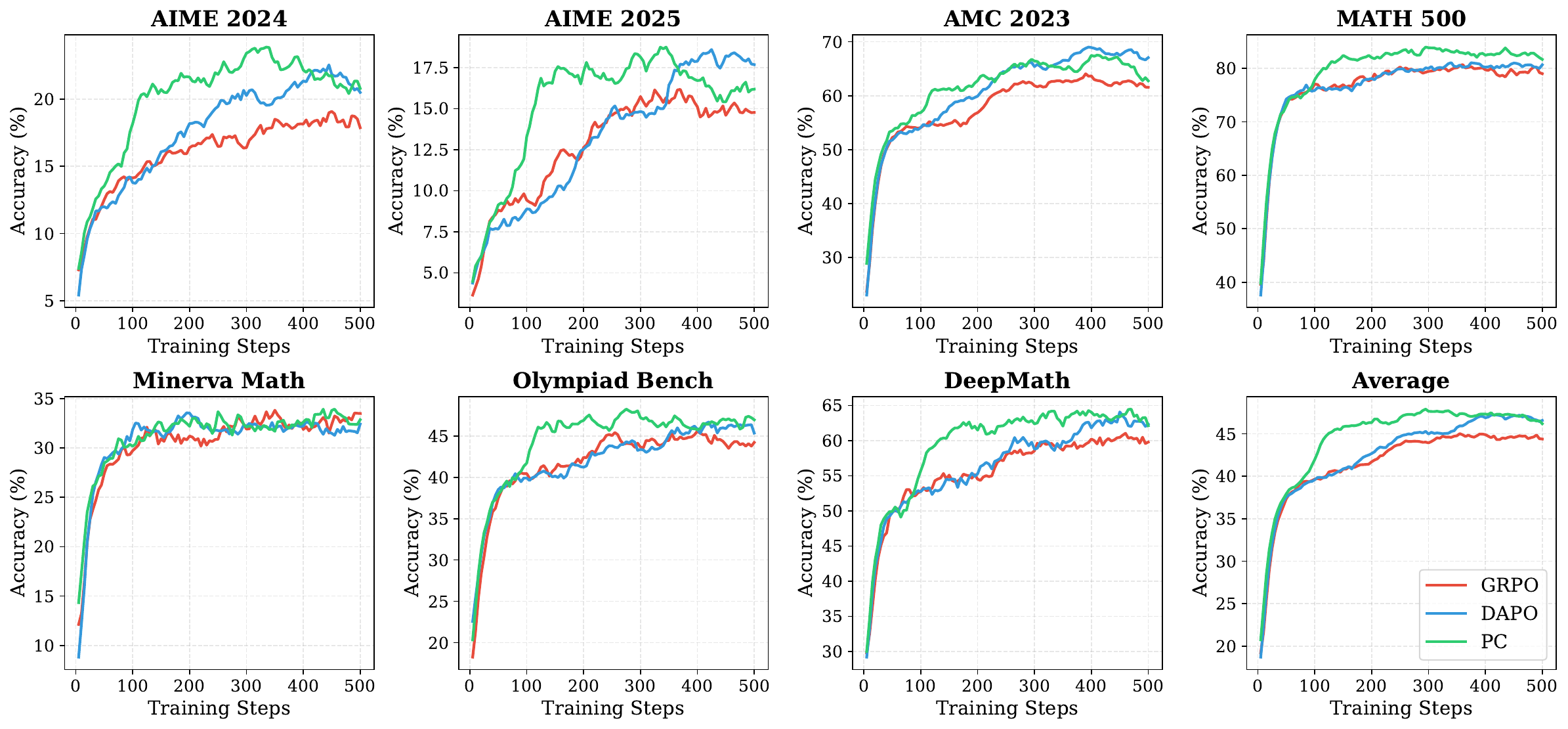}
    \caption{Qwen3-4B, $n{=}16$: accuracy vs training steps.}
    \label{fig:app-4b-16-steps}
\end{figure}

\begin{figure}[H]
    \centering
    \includegraphics[width=\textwidth]{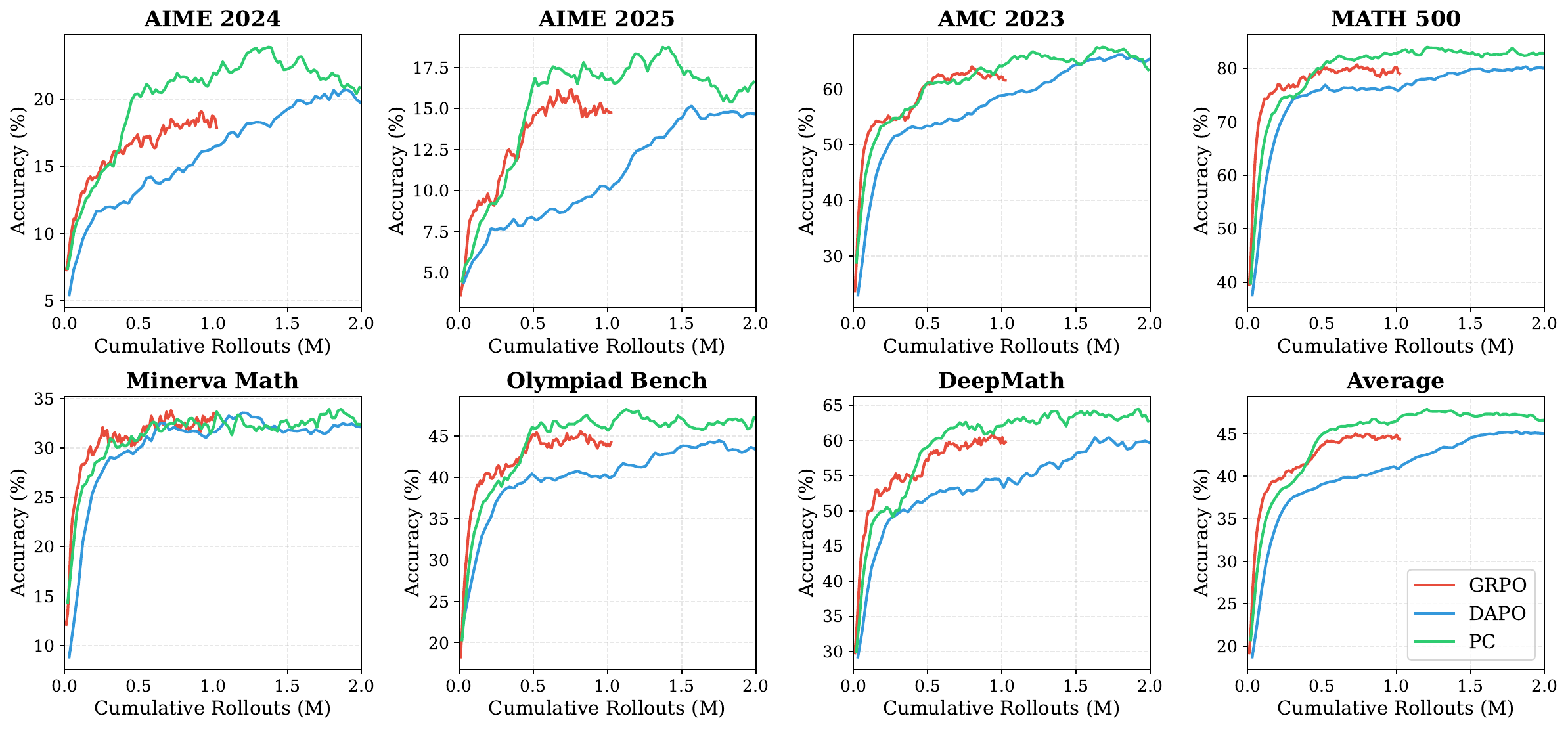}
    \caption{Qwen3-4B, $n{=}16$: accuracy vs cumulative rollouts.}
    \label{fig:app-4b-16-rollouts}
\end{figure}

\subsubsection{Qwen3-8B (Polaris-53K)}

\begin{figure}[H]
    \centering
    \includegraphics[width=\textwidth]{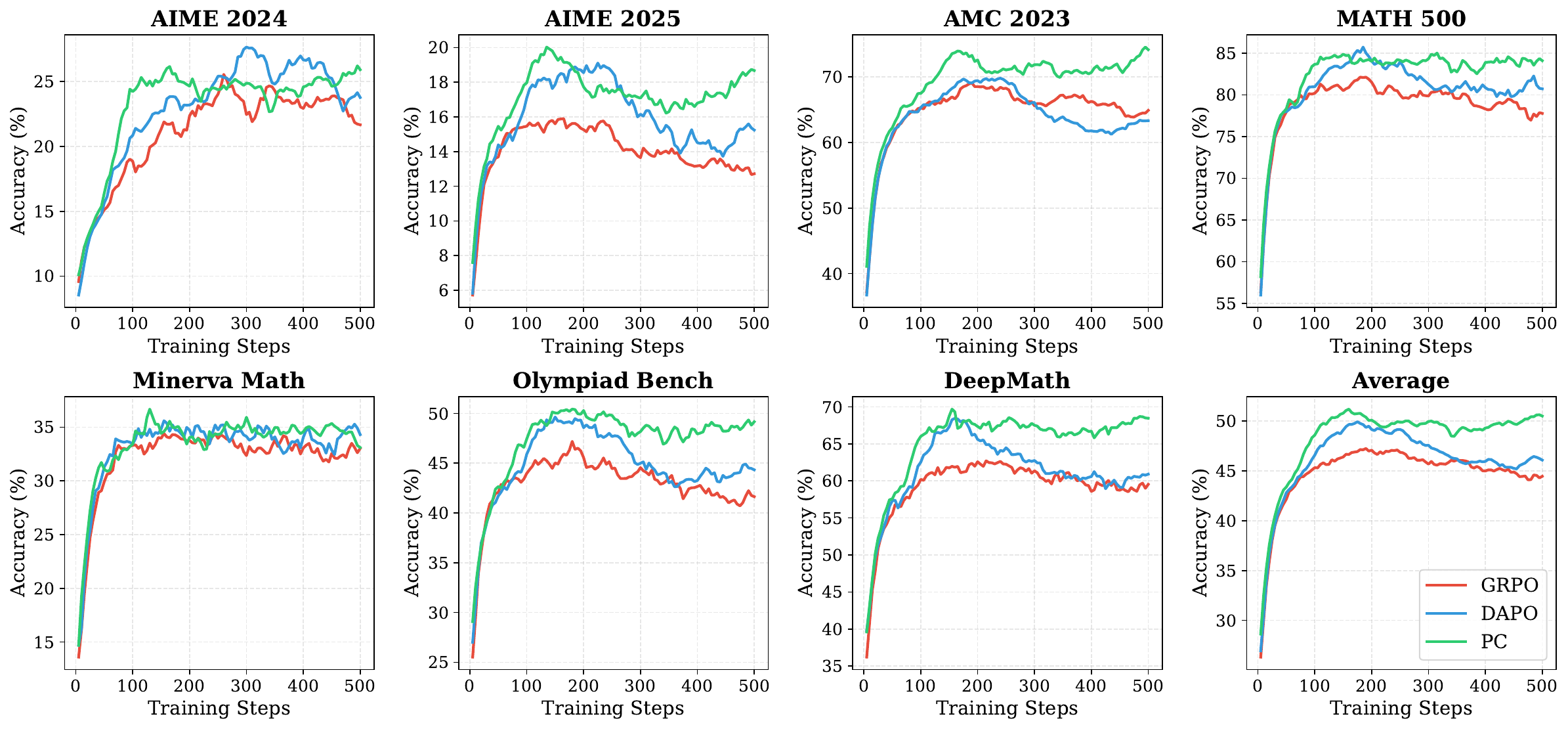}
    \caption{Qwen3-8B, $n{=}64$: accuracy vs training steps.}
    \label{fig:app-8b-64-steps}
\end{figure}

\begin{figure}[H]
    \centering
    \includegraphics[width=\textwidth]{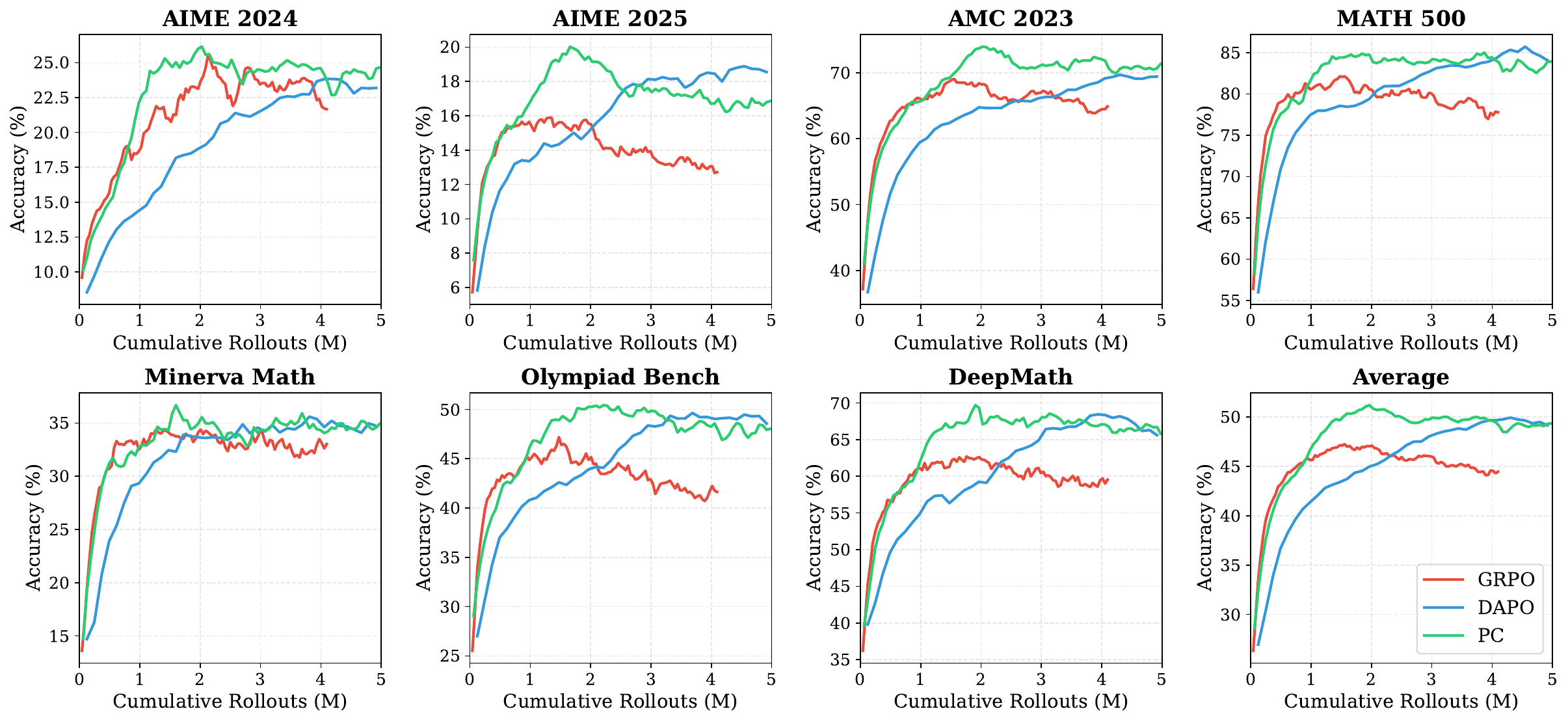}
    \caption{Qwen3-8B, $n{=}64$: accuracy vs cumulative rollouts.}
    \label{fig:app-8b-64-rollouts}
\end{figure}

\begin{figure}[H]
    \centering
    \includegraphics[width=\textwidth]{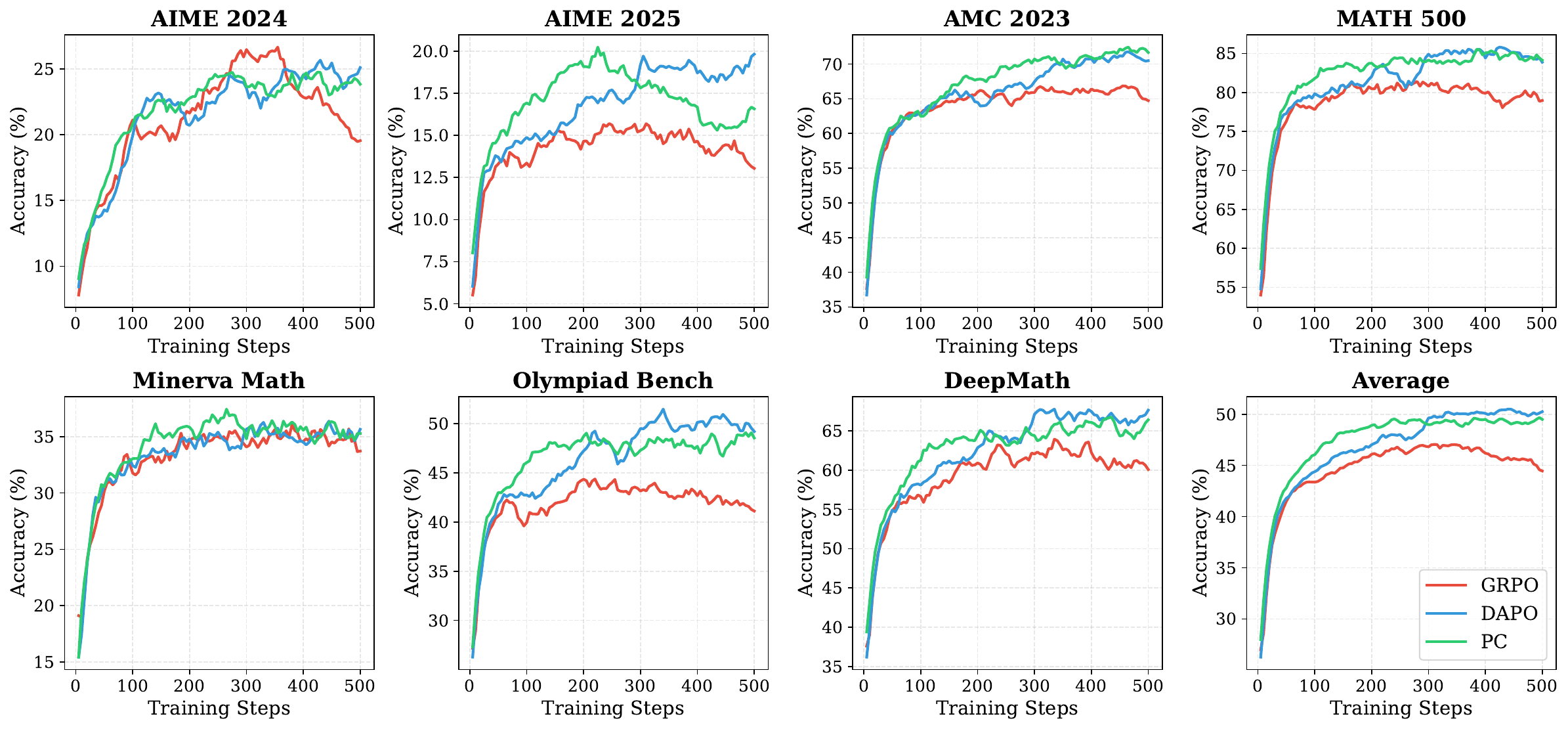}
    \caption{Qwen3-8B, $n{=}16$: accuracy vs training steps.}
    \label{fig:app-8b-16-steps}
\end{figure}

\begin{figure}[H]
    \centering
    \includegraphics[width=\textwidth]{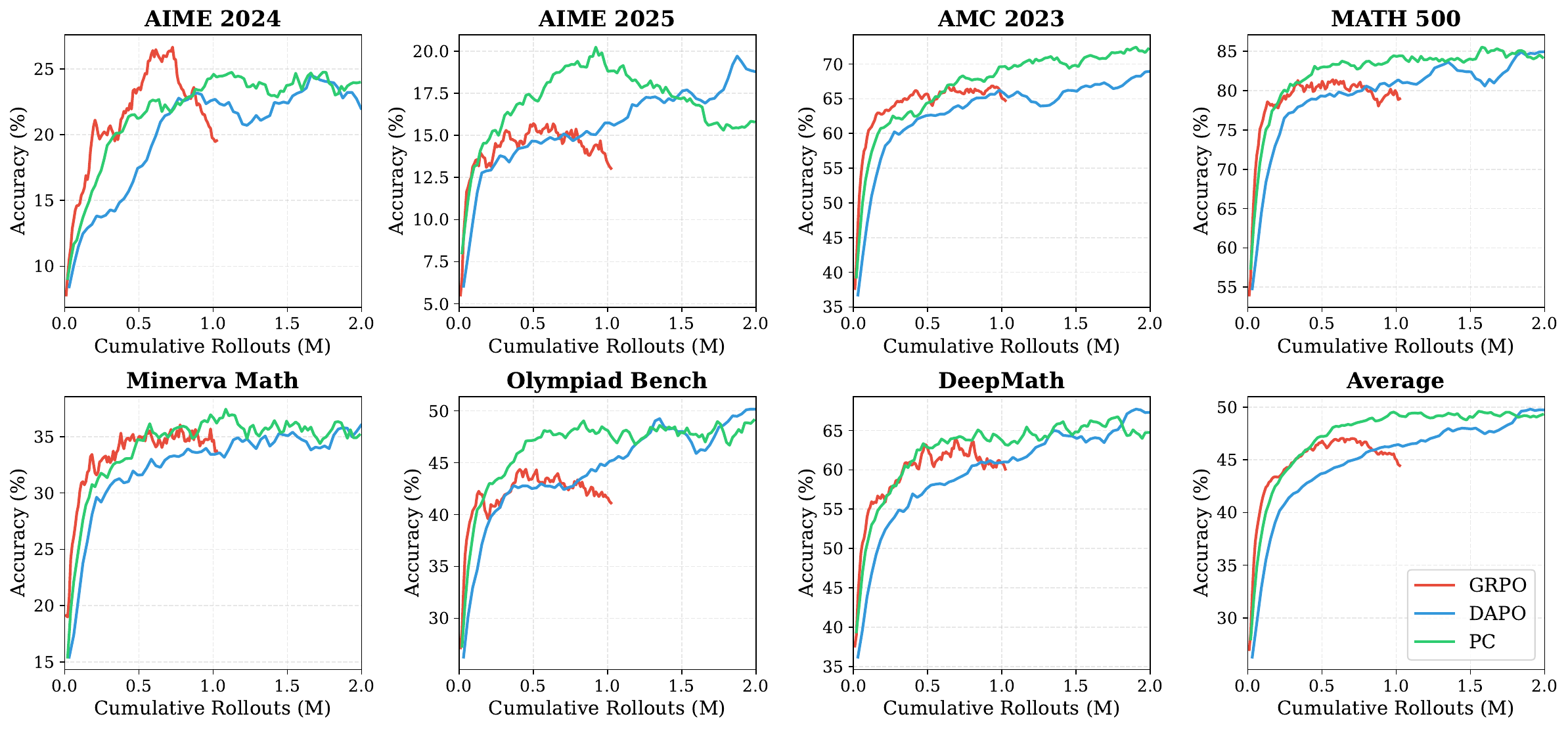}
    \caption{Qwen3-8B, $n{=}16$: accuracy vs cumulative rollouts.}
    \label{fig:app-8b-16-rollouts}
\end{figure}

\subsubsection{Qwen3-14B (Polaris-53K)}

\begin{figure}[H]
    \centering
    \includegraphics[width=\textwidth]{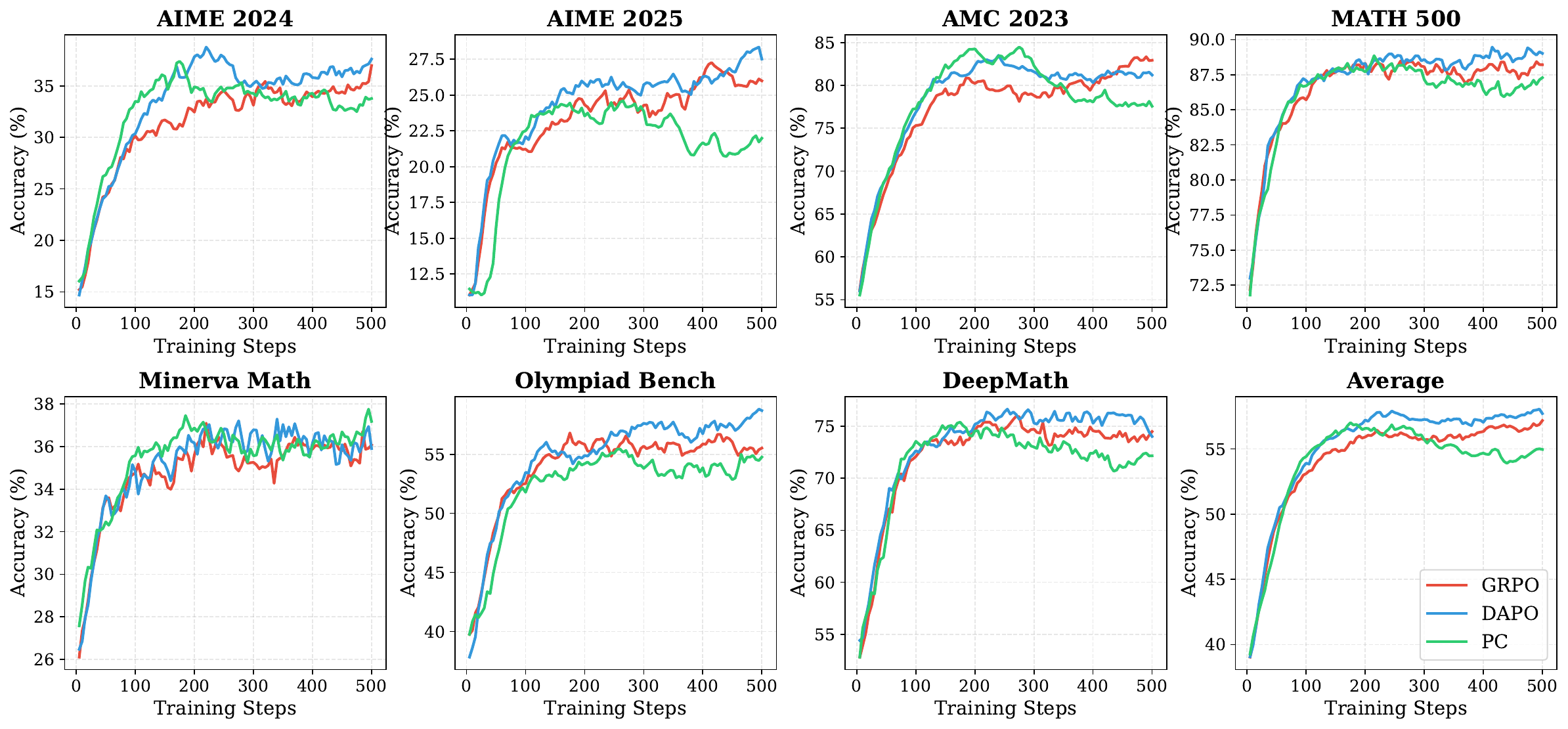}
    \caption{Qwen3-14B, $n{=}64$: accuracy vs training steps.}
    \label{fig:app-14b-64-steps}
\end{figure}

\begin{figure}[H]
    \centering
    \includegraphics[width=\textwidth]{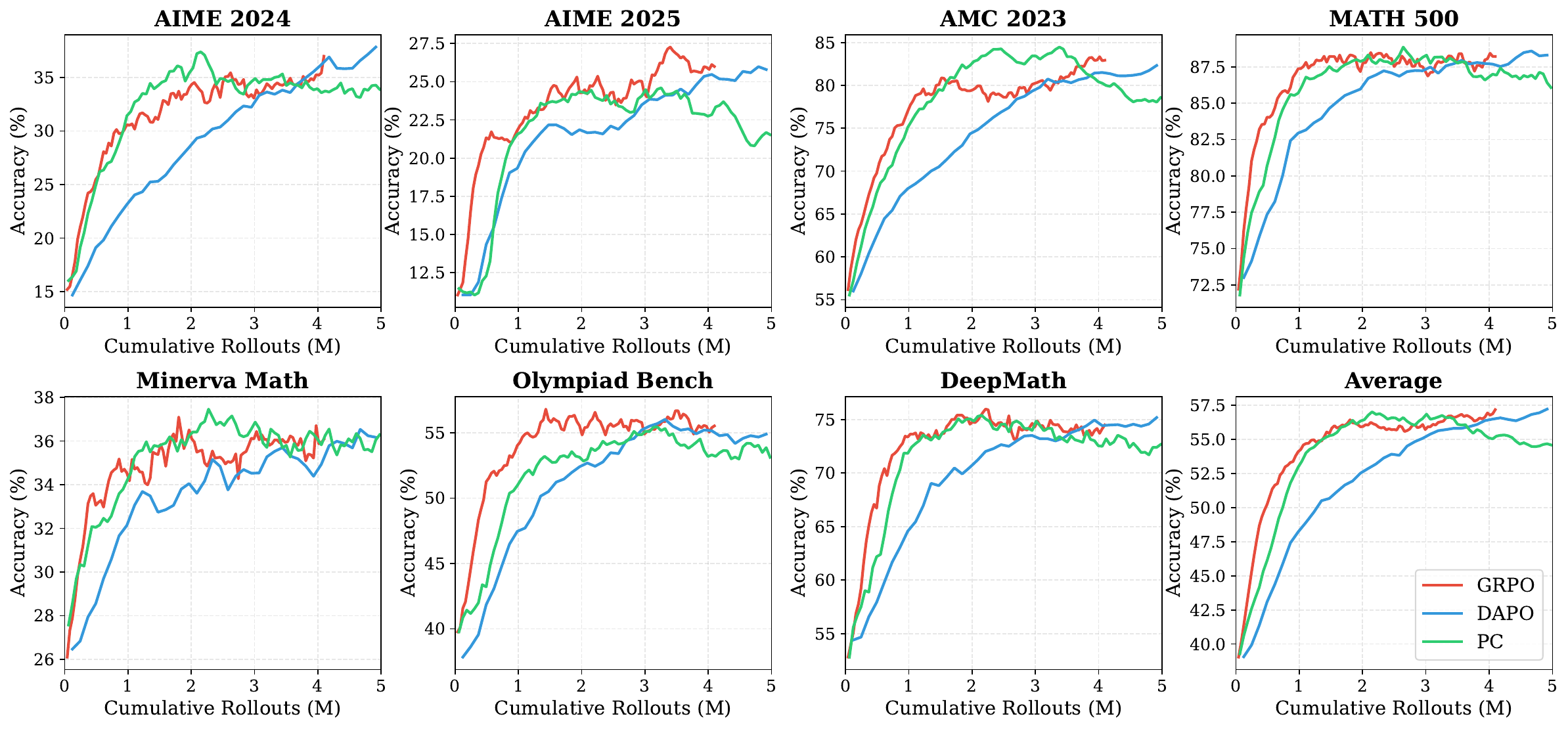}
    \caption{Qwen3-14B, $n{=}64$: accuracy vs cumulative rollouts.}
    \label{fig:app-14b-64-rollouts}
\end{figure}

\begin{figure}[H]
    \centering
    \includegraphics[width=\textwidth]{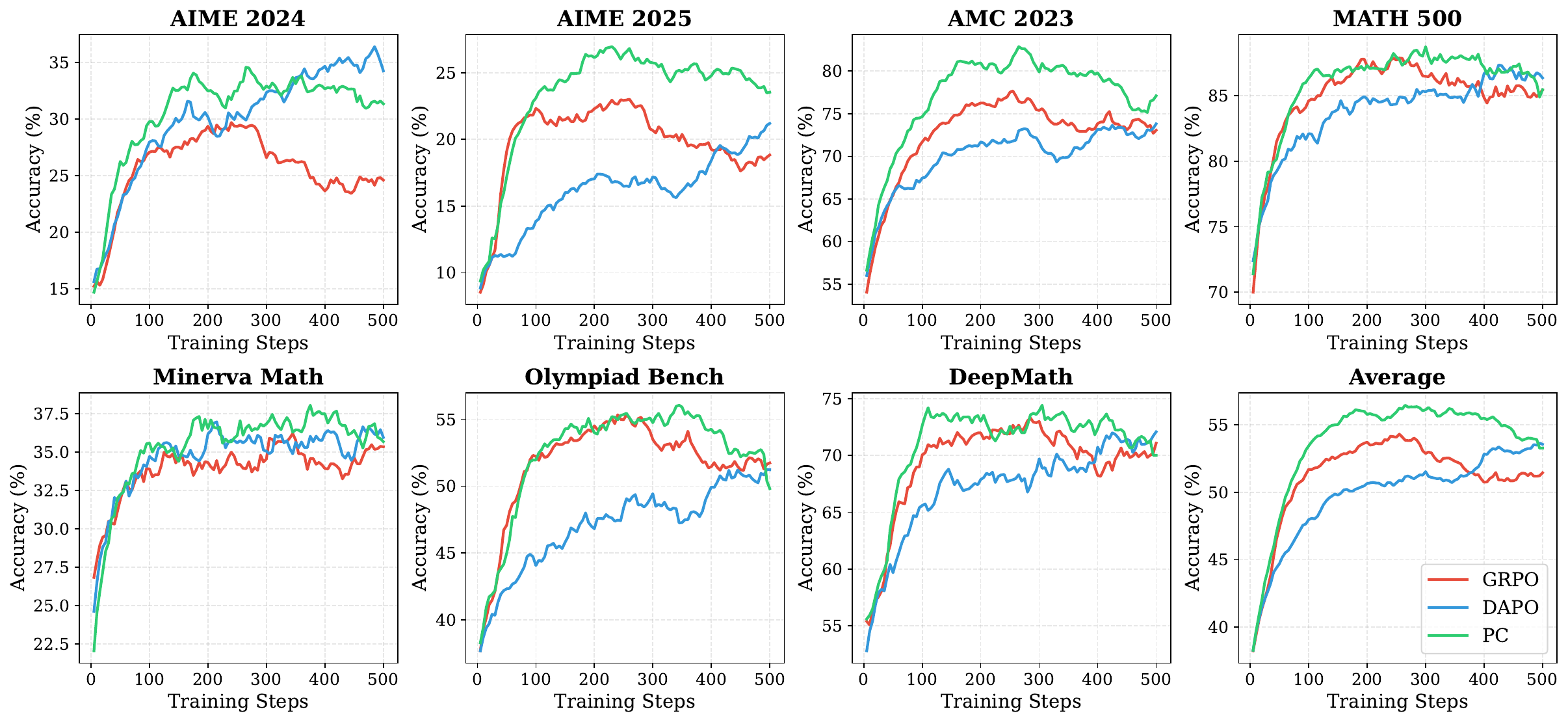}
    \caption{Qwen3-14B, $n{=}16$: accuracy vs training steps.}
    \label{fig:app-14b-16-steps}
\end{figure}

\begin{figure}[H]
    \centering
    \includegraphics[width=\textwidth]{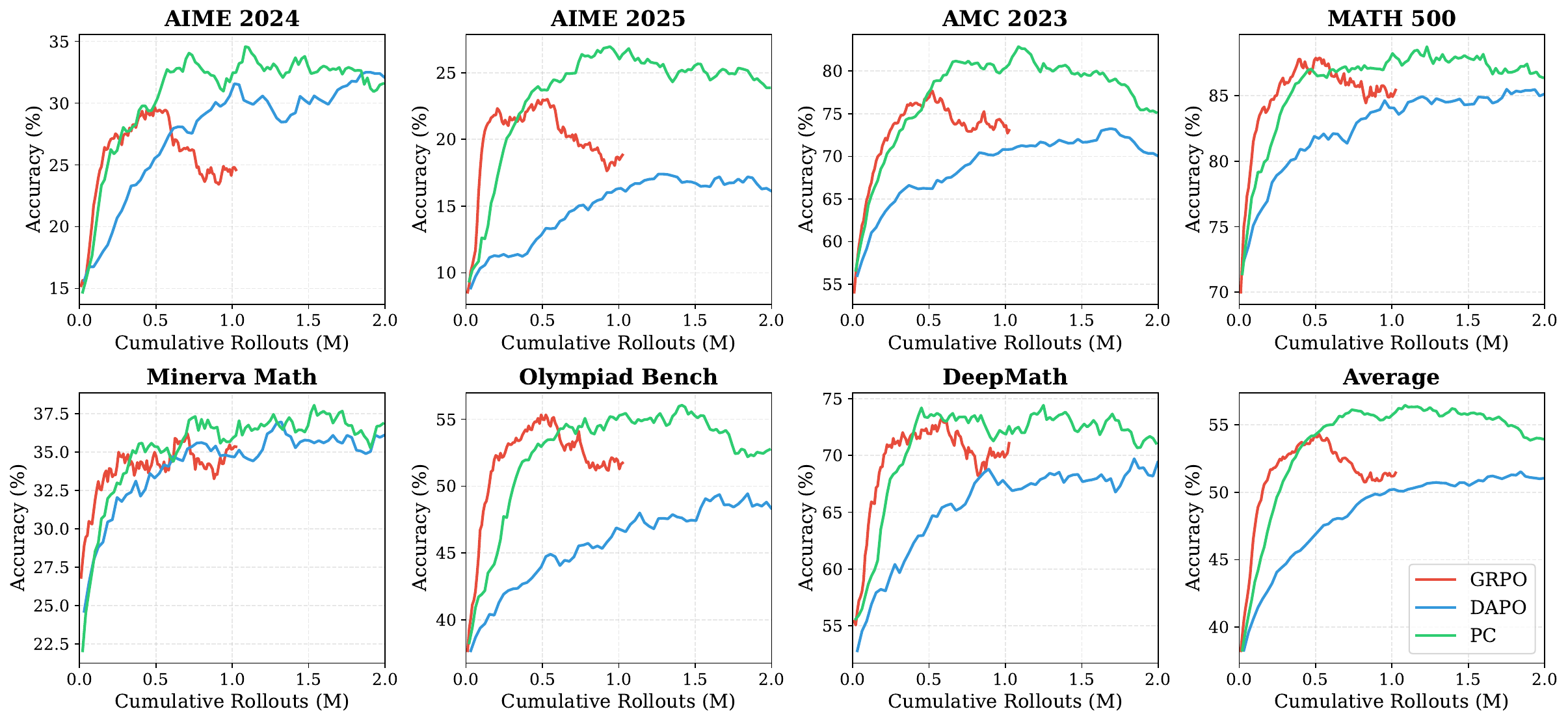}
    \caption{Qwen3-14B, $n{=}16$: accuracy vs cumulative rollouts.}
    \label{fig:app-14b-16-rollouts}
\end{figure}

% \subsection{Eviction Statistics for $p_{\mathrm{solve}}$ Ablation}
%
% \begin{figure}[H]
%     \centering
%     \includegraphics[width=0.8\textwidth]{figures/p_solve_ablation_stats.pdf}
%     \caption{Eviction statistics for different $p_{\mathrm{solve}}$ values. Left: fraction of prompts flagged as too correct per step. Right: cumulative fraction of prompts evicted from the dataset. More aggressive thresholds evict substantially more prompts without improving accuracy.}
%     \label{fig:ablation_p_solve_exclusion}
% \end{figure}

\subsection{Pilot/Commit Allocation under Equal Sampling Cost}

\begin{figure}[H]
    \centering
    \includegraphics[width=\textwidth]{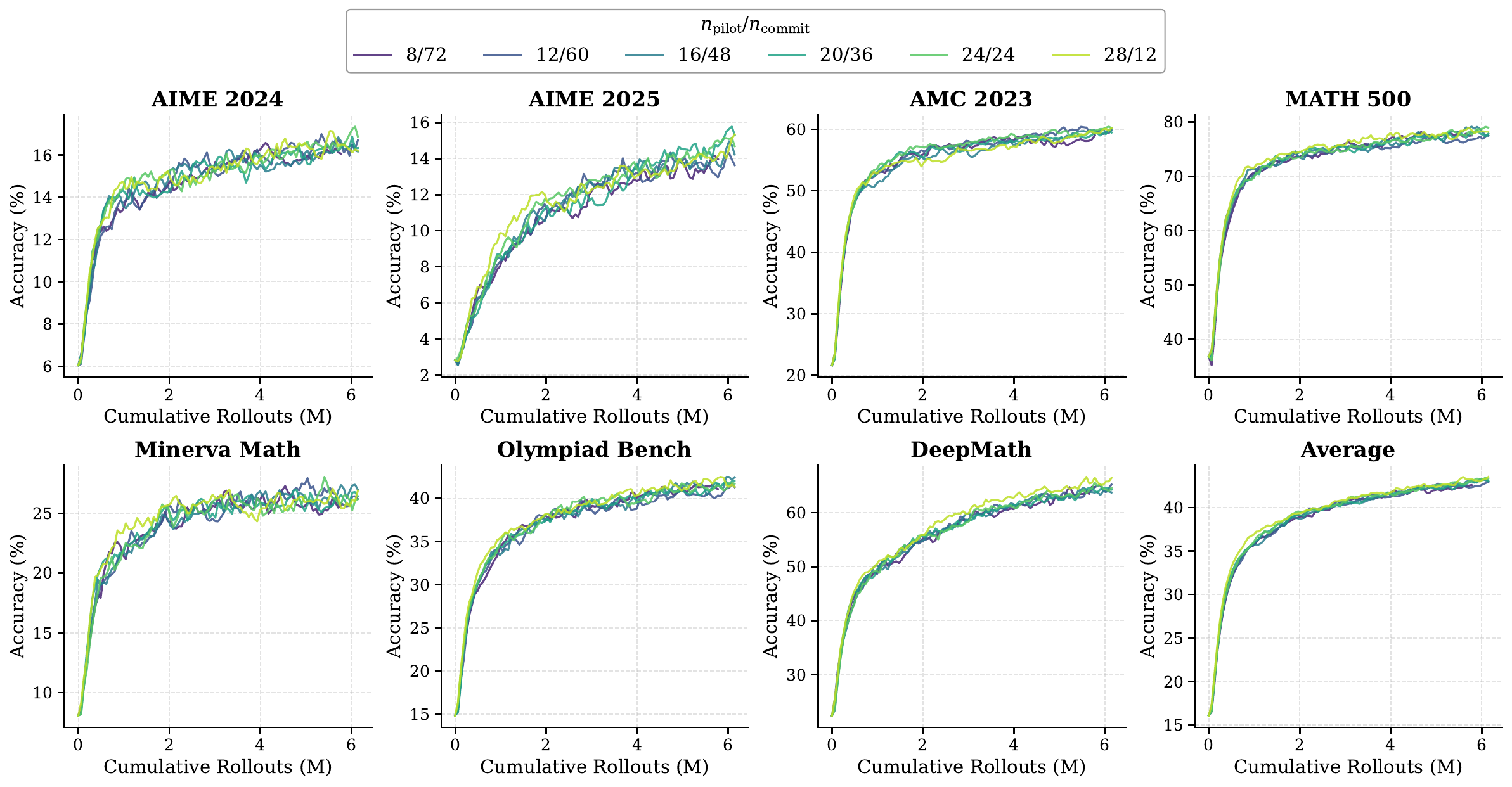}
    \caption{Comparison of Pilot-Commit configurations under equal sampling cost. The number of training rollouts varies across configurations.
    }
    \label{fig:fixed_sample_budget}
\end{figure}

\begin{figure}[H]
    \centering
    \includegraphics[width=\textwidth]{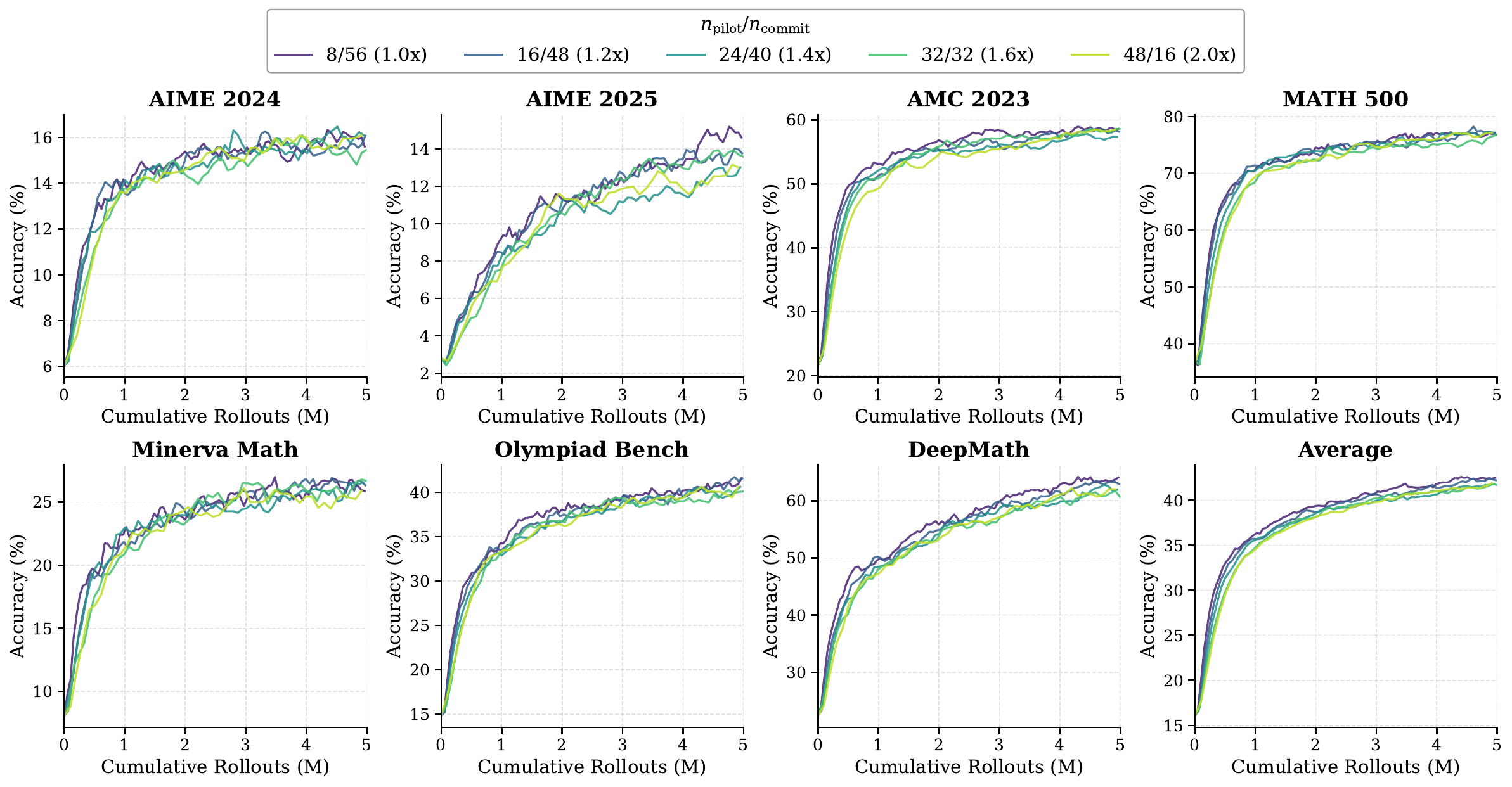}
    \caption{Comparison of Pilot-Commit configurations under equal training cost. While $n_{\mathrm{pilot}} + n_{\mathrm{commit}}$ is held constant, different allocations result in varying sampling costs per update.
    }
    \label{fig:fixed_train_budget}
\end{figure}

\end{document}